%% file: acl_latex.tex
\documentclass[11pt]{article}

\PassOptionsToPackage{table}{xcolor}
\usepackage[final]{acl}

\usepackage{times}
\usepackage{latexsym}

\usepackage[T1]{fontenc}
\usepackage[utf8]{inputenc}

\usepackage{microtype}

\usepackage{inconsolata}

\usepackage{graphicx}
\usepackage{algorithm}
\usepackage{algpseudocode}
\usepackage{booktabs}
\usepackage{url}
\usepackage{makecell}
\usepackage{pifont,tikz}
\usepackage{amsmath, amssymb, amsfonts}
\usepackage{listings}
\usepackage{threeparttable}
\usepackage{wrapfig}
\usepackage{enumitem}
\usepackage{subcaption}
\usepackage{tcolorbox}
\usepackage{amsthm}
\newtheorem{proposition}{Proposition}
\tcbuselibrary{breakable}
\definecolor{improve}{RGB}{220, 240, 220}  % soft green: signals gain
\newcommand{\improved}[1]{\cellcolor{improve}#1}

\input{macros}

\title{Time is Not a Label: Continuous Phase Rotation for Temporal Knowledge Graphs and Agentic Memory}

\author{
    Weixian Waylon Li\thanks{This work was done during an internship of Weixian Waylon Li at LIGHTSPEED (UK).}\\
    University of Edinburgh \\
    {\small \texttt{waylon.li@ed.ac.uk}} \\\And
    Jiaxin Zhang\thanks{Corresponding authors.} \\
    LIGHTSPEED \\
    {\small \texttt{jiaxijzhang@global.tencent.com}} \\ \And 
    Xianan Jim Yang \\
    University of St Andrews \\
    {\small \texttt{xy60@st-andrews.ac.uk}} \\\AND
    Tiejun Ma\footnotemark[2] \\
    University of Edinburgh \\
    {\small \texttt{tiejun.ma@ed.ac.uk}} \\ \And
    Yiwen Guo\footnotemark[2]  \\
    Independent Researcher \\
    {\small \texttt{guoyiwen89@gmail.com}}
}

\begin{document}
\maketitle
\begin{abstract}
  Structured memory representations such as knowledge graphs are central to autonomous agents and other long-lived systems. However, most existing approaches model time as discrete metadata, either sorting by recency (burying old-yet-permanent knowledge), simply overwriting outdated facts, or requiring an expensive LLM call at every ingestion step, leaving them unable to distinguish persistent facts from evolving ones. To address this, we introduce \ourapproach\footnote{Codes available at: \url{https://github.com/Tencent/RoMem}}, a drop-in temporal knowledge graph module for structured memory systems, applicable to agentic memory and beyond. A pretrained \textit{Semantic Speed Gate} maps each relation's text embedding to a volatility score, learning from data that evolving relations (e.g., ``president of'') should rotate fast while persistent ones (e.g., ``born in'') should remain stable. Combined with continuous phase rotation, this enables \textit{geometric shadowing}: obsolete facts are rotated out of phase in complex vector space, so temporally correct facts naturally outrank contradictions without deletion. On temporal knowledge graph completion, \ourapproach achieves state-of-the-art results on ICEWS05-15 (72.6 MRR). Applied to agentic memory, it delivers $2{\sim}3\times$ MRR and answer accuracy on temporal reasoning (MultiTQ), dominates hybrid benchmark (LoCoMo), preserves static memory with zero degradation (DMR-MSC), and generalises zero-shot to unseen financial domains (FinTMMBench). 
\end{abstract}

\input{sections/introduction}

\input{sections/related-work}

\input{sections/methodology}

\input{sections/experiments}

\input{sections/conclusion}

\input{sections/limitation}

\input{sections/acknowledgements}

\bibliography{custom}

\appendix

\input{appendices/appendix-master}

\end{document}

%% file: macros.tex
\usepackage{xspace} 

\newcommand{\ourapproach}{\textsc{RoMem}\xspace}

%% file: sections/introduction.tex
\section{Introduction}
\label{sec:intro}

Structured memory representations such as knowledge graphs have become widely adopted as the long-term memory substrate for agentic systems \citep{10387715,chhikara2025mem0buildingproductionreadyai,gutiérrez2024hipporag,gutierrez2025hipporag2,rasmussen2025zeptemporalknowledgegraph,huang2026licomemorylightweightcognitiveagentic,jiang2026magmamultigraphbasedagentic}, providing unbounded, structured, and verifiable memory that decouples storage from the LLM.

However, a fundamental challenge is that, most graph-based systems model time as \textit{discrete metadata}, a timestamp that cannot encode whether a relation is permanent or transient \citep{tang2026memoryforgesynthesizelifelongmemory}.
The real world is dynamic~\citep{10.24963/ijcai.2023/734}: executive boards shift, borders change, and markets fluctuate.
When temporal conflicts arise (e.g., ``Obama is president'' vs.\ ``Biden is president''), current systems resort to three workarounds: (i)~\textbf{destructive overwriting}, which erases historical context~\citep{xu2025amemagenticmemoryllm,gutiérrez2024hipporag}; (ii)~\textbf{LLM arbitration}, which requires a language model call at every ingestion step to predict symbolic UPDATE/DELETE commands~\citep{chhikara2025mem0buildingproductionreadyai,rasmussen2025zeptemporalknowledgegraph,yan2025memoryr1enhancinglargelanguage}; or (iii)~\textbf{recency sorting}, which ranks facts by timestamp to surface the latest version.
Each has notable limitations.
Destructive overwriting permanently loses historical context.
LLM arbitration may suit short-term conversational memory, but becomes infeasible when scaling to long-term memory with millions of facts.
Recency sorting, the most common workaround, appears to work until it silently buries static knowledge: a recency-based system ranks the decades-old fact \texttt{(Obama, born\_in, Hawaii)} below fresher but irrelevant entries.
Disabling recency bias leaves temporal conflicts unresolved, confusing the downstream LLM~\citep{liu-etal-2024-lost}.
We call this the \textit{static-dynamic dilemma}: discrete metadata treats all relations identically and cannot resolve temporal conflicts without sacrificing static knowledge.

\begin{figure*}[htbp]
    \centering
    \includegraphics[width=1.0\linewidth]{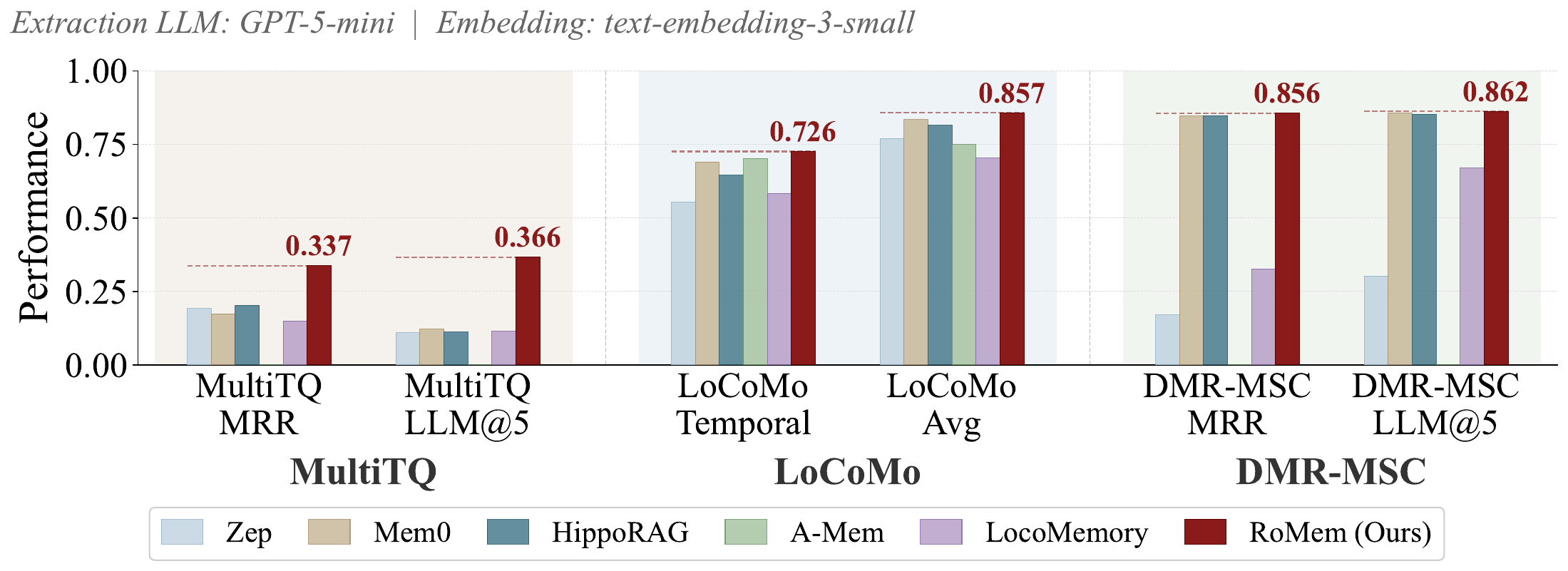}
    \caption{\textbf{Performance overview.} \ourapproach adds temporal reranking to the same HippoRAG graph-construction and semantic-retrieval pipeline. It improves temporal and hybrid retrieval while preserving performance on largely non-temporal memory. Table~\ref{tab:main_results_grid} reports results for both the closed- and open-source configurations.}
    \label{fig:results_overview}
\end{figure*}

To this end, we introduce \ourapproach, a temporal reasoning module for graph-based agentic memory that internalises time as a continuous geometric operator.
Rather than building a new memory system, \ourapproach provides a temporal engine for the knowledge graph component: it learns to distinguish static from dynamic relations zero-shot and resolves conflicts through geometry rather than database operations.
We achieve this through two mechanisms:
(1)~\textbf{Continuous Geometric Shadowing}, which models time as a functional phase shift in complex vector space, rotating dynamic facts out of alignment as they become obsolete while keeping static facts locked in phase; and (2) a \textbf{Semantic Speed Gate} that estimates relational volatility from text embeddings, outputting a per-relation scalar $\alpha_r \in (0,1)$ that controls rotation speed, so static relations ($\alpha_r \!\approx\! 0$) remain stable while dynamic ones ($\alpha_r \!\approx\! 1$) rotate to shadow obsolete facts.
The memory remains strictly append-only, yet the LLM receives a clean, unambiguous context window driven entirely by geometric proximity.

Our main contributions are as follows:
\begin{itemize}[leftmargin=1.5em, itemsep=4pt, topsep=4pt, parsep=0pt]
    \item We formalise the \textit{static-dynamic dilemma} in graph-based agentic memory, showing that discrete timestamp metadata treats all relations identically, preventing temporal conflict resolution without sacrificing static knowledge.
    \item We formulate temporal conflict resolution as continuous geometric shadowing in complex vector space, replacing destructive database updates and per-ingestion LLM calls with an append-only architecture.
    \item We introduce a Semantic Speed Gate that addresses this dilemma by learning relational volatility from text embeddings, generalising zero-shot to unseen relations and domains without manual annotation.
    \item We demonstrate that \ourapproach achieves SOTA temporal knowledge graph completion on ICEWS05-15 (72.6 MRR) and, applied to agentic memory, delivers $2{\sim}3\times$ MRR and accuracy on temporal reasoning (MultiTQ), dominates hybrid tasks (LoCoMo), preserves static memory (DMR-MSC), and generalises zero-shot to unseen financial domains (FinTMMBench). The performance overview is in Figure~\ref{fig:results_overview}.

\end{itemize}

%% file: sections/related-work.tex
\section{Related Work}
\label{sec:related-work}

\paragraph{Agentic Memory Paradigms.}
Graph-based memory is widely adopted for long-term agent knowledge, with frameworks such as Mem0~\citep{chhikara2025mem0buildingproductionreadyai}, HippoRAG~\citep{gutiérrez2024hipporag,gutierrez2025hipporag2}, Zep~\citep{rasmussen2025zeptemporalknowledgegraph}, DialogGSR~\citep{park-etal-2024-generative-dialoggsr}, LicoMemory~\citep{huang2026licomemorylightweightcognitiveagentic}, and DescGraph~\citep{hu2026doesmemoryneedgraphs} offering scalable, controllable retrieval.
Parametric approaches~\citep{yao2024retroformer,zhang2025agentlearningearlyexperience} require costly retraining and offer less transparent retrieval~\citep{tois-memory-survey,hu2026memoryageaiagents}.
Context engineering strategies, including compression~\citep{ye2025agentfoldlonghorizonwebagents}, dynamic context management~\citep{yu2025memagentreshapinglongcontextllm,zhou2025mem1learningsynergizememory,yan2025memoryr1enhancinglargelanguage,salama-etal-2025-meminsight,yan2026memdefraglatentmemorydefragmentation}, and reflective evolution~\citep{liang2025sageselfevolvingagentsreflectivememoryaugmented,zhou2025mementofinetuningllmagents}, are typically limited to short-term conversational state.

\paragraph{Temporal Gap in Memory.}
Existing memory systems lack a native mechanism for managing changing facts.
Most treat memory as a static snapshot and resort to three workarounds: (i)~destructive overwriting that permanently erases historical context \citep{xu2025amemagenticmemoryllm,gutiérrez2024hipporag,yu2025memagentreshapinglongcontextllm}; (ii)~LLM-driven arbitration that predicts UPDATE or DELETE actions during ingestion \citep{chhikara2025mem0buildingproductionreadyai,rasmussen2025zeptemporalknowledgegraph,yan2025memoryr1enhancinglargelanguage}; or (iii)~recency-based metadata sorting, which inadvertently buries old-yet-permanent facts \citep{jiang2026magmamultigraphbasedagentic}.
LLM arbitration and \ourapproach operate at different stages. The former updates memory during ingestion, whereas \ourapproach retains the full history and ranks facts at query time.
Notably, none of these approaches distinguish \emph{relational volatility}: they apply the same temporal policy to static facts (``born in'') and dynamic ones (``president of'').

\paragraph{Temporal Rotation and Semantic Transfer.}
TKG embedding methods such as TeRo~\citep{xu-etal-2020-tero} and ChronoR~\citep{sadeghian2021chronor} already represent time through rotation; temporal rotation itself is therefore not our contribution.
These methods learn a separate representation for each timestamp at a fixed resolution. \ourapproach instead defines time through the relation-conditioned function $\boldsymbol{\theta}_r(\tau)$, so it can score an unseen timestamp without learning a new lookup embedding. We use this continuous operator to resolve conflicting facts while retaining the full history in agentic memory.
Relation semantics have also been used for zero-shot transfer. For example, zrLLM~\citep{ding-etal-2024-zrllm} constructs relation representations with an LLM for zero-shot TKG forecasting. In contrast, our Semantic Speed Gate maps relation text to a scalar volatility that controls rotation speed. The entity and relation embeddings and the global temporal spectrum are then learned on the target graph.

%% file: sections/methodology.tex
\section{Methodology: \ourapproach}
\label{sec:method}

\ourapproach (Figure~\ref{fig:romem-overview}) models temporal conflict resolution as a geometric physical law within the knowledge graph embedding space, replacing discrete memory management with continuous phase rotation.
The memory module is strictly \emph{append-only}: contradictions co-exist in memory and are resolved at query time through \textit{geometric shadowing}, where the temporally aligned fact naturally outranks obsolete ones via phase proximity.
Because time is a continuous function rather than a discrete label, the system natively supports historical retrieval and zero-shot evaluation of unseen dates.

\begin{figure*}
    \centering
    \includegraphics[width=1.0\linewidth]{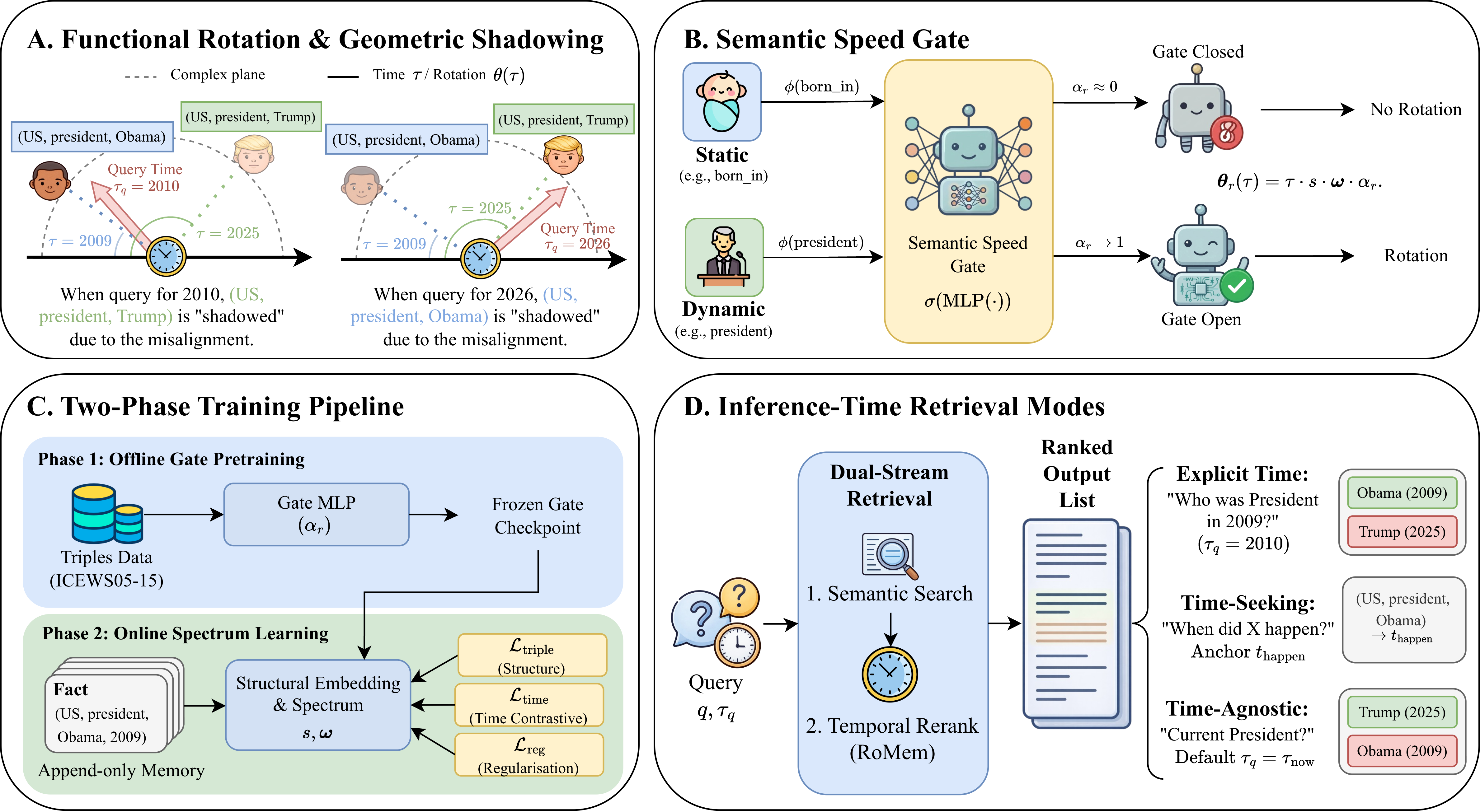}
    \caption{\textbf{Overview of the \ourapproach Architecture.} The framework consists of four stages:  \emph{(A) Functional Rotation} (\S\ref{sec:rotation}) applies geometric phase shifts to obsolete facts; \emph{(B) Semantic Speed Gate} (\S\ref{sec:gate}) determines relational volatility from text embeddings; \emph{(C) Two-Phase Training} (\S\ref{sec:training}) pretrains the gate and learns the temporal spectrum; \emph{(D) Inference-Time Retrieval} (\S\ref{sec:inference}) resolves contradictions via geometric shadowing.}
    \label{fig:romem-overview}
\end{figure*}

\subsection{Problem Setting and Memory Design}
We process a stream of textual episodes $\{d_i\}_{i=1}^{N}$ to answer queries $q$. Each episode yields relational facts via temporal Open Information Extraction (OpenIE):
$f=(h,r,t)$ where $h,t \in \mathcal{E},\ r \in \mathcal{R}$,
with a valid time $t_{\mathrm{happen}}(f)$ extracted from text and an observation time $t_{\mathrm{obs}}(f)$ at ingestion. 
If no valid time is present, $t_{\mathrm{happen}}$ remains unknown.
We maintain an append-only memory state $\mathcal{M}$ where contradictions co-exist: $m=(f,t_{\mathrm{happen}}(f),t_{\mathrm{obs}}(f),\mathrm{src})$, where $\mathrm{src} \in \{d_i\}_{i=1}^{N}$.
We store dense embeddings for passages and entities using a text encoder $\phi(\cdot)$ and build a heterogeneous graph $G=(V,E)$ induced by facts.
Extraction prompts are provided in Appendix~\ref{app:prompts}.

\subsection{Core Insight: Why Geometry Solves Temporal Conflicts}
\label{sec:insight}

The most straightforward approach to temporal memory is to store a timestamp with each fact and sort by recency at retrieval time. However, as discussed in \S\ref{sec:intro}, this metadata-based approach faces the \textit{static-dynamic dilemma}: it treats all relations identically, unable to distinguish ``born in'' (static) from ``president of'' (dynamic). To resolve this dilemma, we employ a Temporal Knowledge Graph Embedding (TKGE) that internalises time as a continuous geometric operator rather than a discrete metadata field. \ourapproach exploits the inductive power of TKG representational learning~\citep{10.24963/ijcai.2023/734} to resolve temporal conflicts natively in vector space.

Our central idea is to model time as a continuous geometric rotation, aligning with evidence from cognitive neuroscience that the mammalian hippocampus encodes time as continuous geometric trajectories rather than discrete timestamps \citep{eichenbaum2014time,howard2014unified}. Consider a ``clock hand'' analogy: the entity embedding for ``Donald Trump'' at $\tau=2025$ is rotated to a phase angle aligned with the ``President'' relation (pointing to 12 o'clock). As time flows to $\tau=2010$, the vector continuously rotates away from this alignment, reducing the retrieval score for ``Donald Trump'' while simultaneously aligning with ``Barack Obama''.
The most temporally relevant fact naturally \emph{shadows} obsolete ones through geometric proximity, without deletion. Crucially, the Semantic Speed Gate (\S\ref{sec:gate}) controls this rotation per-relation: static facts do not rotate (and therefore are never buried by recency), while dynamic facts rotate rapidly (resolving temporal conflicts geometrically).

Beyond the metadata-based approach, this design also overcomes two limitations of prior \emph{temporal embedding} methods.
First, \textbf{additive models} (e.g., T-TransE~\citep{10.1145/3184558.3191639}, HyTE~\citep{dasgupta-etal-2018-hyte}) treat time as a linear bias added to structural embeddings.
This suffers from \textit{additive decoupling}: a strong structural affinity (e.g., for popular entities) can overpower the temporal penalty so that anachronistic facts are retrieved based on popularity.
Our multiplicative rotation couples temporal alignment with structural affinity, helping the model demote facts that are structurally plausible but temporally mismatched.
Second, \textbf{discrete rotation models} (e.g., ChronoR \citep{sadeghian2021chronor}, TeRo \citep{xu-etal-2020-tero}) learn a separate embedding vector $\boldsymbol{\tau}_t$ for every observed timestamp.
Lacking a continuous functional bridge, this design leaves blind spots between observed timestamps.
Our functional definition $\boldsymbol{\theta}(\tau)$ instead treats time as a continuous geometric variable, allowing the model to score timestamps that were not observed during training. Appendix~\ref{app:math-analysis-temporal-interpolation} proves monotonic interpolation under a sufficient condition.

\subsection{Functional Rotation Mechanism}
\label{sec:rotation}

We embed entities and relations in $\mathbb{R}^{2d}$, interpreted as complex vectors in $\mathbb{C}^{d}$. 
A scalar time $\tau$ acts as a rotation operator in the unitary group $U(1)^d$ via the operator $\mathrm{Rot}(\mathbf{x}, \boldsymbol{\theta})$, which applies an element-wise phase shift $\mathbf{x} \odot e^{i\boldsymbol{\theta}}$. 
We define the relation-specific rotation angle as:
\begin{equation}
\boldsymbol{\theta}_r(\tau) = s \cdot \alpha_r \cdot \tau \cdot \boldsymbol{\omega},
\label{eq:theta_r}
\end{equation}
where $s \in \mathbb{R}^{+}$ is a global time scale parameter. We initialise $s$ with a day-level prior ($s_0 = 1/86400$) and learn it jointly with the target graph.
$\alpha_r \in (0,1)$ is the semantic speed gate (\S\ref{sec:gate}), $\tau$ is the continuous timestamp, and $\boldsymbol{\omega} \in \mathbb{R}^{k \times d}$ is a learnable inverse-frequency tensor. Under a flattened index $i$, its elements are initialised as $\omega_i=b^{-i/(kd)}$, where the base $b$ is initialised at $10{,}000$ and learned during training.

We build upon the multi-component bilinear architecture of ChronoR~\citep{sadeghian2021chronor}, replacing its discrete timestamp lookup with our functional time definition. Each entity has $k$ complex components (we use $k{=}3$, following ChronoR), each in $\mathbb{C}^d$, with relation embeddings $\mathbf{w}_r, \hat{\mathbf{w}}_r \in \mathbb{R}^{k \times 2d}$ for forward and inverse semantics, respectively. The scoring function is:
\begin{equation}
\mathbf{v}_r^c(\mathbf{e},\tau) = \mathrm{Rot}(\mathbf{e}^c, \boldsymbol{\theta}_r^c(\tau))
\end{equation}
\begin{equation}
\tilde{\mathbf{v}}_r^c(\mathbf{e},\tau) = \mathbf{v}_r^c(\mathbf{e},\tau) \odot \mathbf{w}_r^c \odot \hat{\mathbf{w}}_r^c
\end{equation}
\begin{equation}
\label{eq:scoring-function}
s_{\mathrm{kge}}\bigl((h,r,t)\mid \tau\bigr) =
\sum_{c=1}^{k} \bigl\langle \tilde{\mathbf{v}}_r^c(\mathbf{e}_h,\tau),\; \mathbf{v}_r^c(\mathbf{e}_t,\tau) \bigr\rangle
\end{equation}
Since unitary rotation preserves the vector modulus and only shifts the phase, an invalid timestamp rotates the fact out of alignment rather than merely penalising its magnitude. 
The efficient 1-vs-$N$ retrieval reformulation and a simplified DistMult variant are detailed in Appendix~\ref{app:method-details}.

\subsection{Semantic Speed Gate}
\label{sec:gate}

A fundamental challenge in applying KGE to OpenIE is relational diversity: OpenIE yields thousands of surface forms (e.g., ``married to'', ``wedded to'', ``spouse of'') for identical relations. Methods that learn a fixed parameter per string cannot generalise across linguistic variations. 
Moreover, relations have distinct temporal natures: ``born in'' is static while ``visiting'' is dynamic. 
We introduce a Semantic Speed Gate that derives rotation velocity from the relation's text embedding $\phi(r)$: $\alpha_r = \sigma\bigl(\mathrm{MLP}(\phi(r))\bigr) \in (0,1)$.
This achieves \textit{zero-shot temporal transfer}: if the model learns that ``married'' implies stability ($\alpha \approx 0$), it automatically stabilises unseen relations like ``wedded'' because their embeddings lie close in semantic space.

The model does not use static/dynamic relation labels. Instead, it learns from temporal transitions.
For \textbf{dynamic relations} where the tail entity changes over time (e.g., \texttt{president\_of}), changed counterparts provide a high-volatility signal and drive $\alpha_r \to 1$.
For \textbf{persistent slots}, unchanged counterparts provide a low-volatility signal and drive $\alpha_r \to 0$.
This functions as a ``temporal clutch'': static facts ($\alpha_r \to 0$) are locked in alignment, while dynamic facts ($\alpha_r \to 1$) rotate to shadow obsolete contradictions.

\subsection{Two-Phase Training}
\label{sec:training}

A central design challenge is decoupling the semantic gate $\alpha_r$ from the global time spectrum $(s, \boldsymbol{\omega})$.
Joint training on a single dataset causes two failure modes: (i)~sparse datasets lack sufficient competing facts to provide a learning signal for $\alpha_r$, causing gate collapse; and (ii)~temporal discrimination objectives (Equation~\eqref{eq:time-contrastive-loss}) treat alternative timestamps as negative samples, which incorrectly penalises the infinite validity of static relations and forces $\alpha_r$ away from the desired zero state.
We address these issues with a two-phase training procedure since $\alpha_r$ depends exclusively on relation text embeddings.

\paragraph{Phase 1: Offline Gate Pretraining.}
We construct a self-supervised dataset of temporal transitions from ICEWS05-15~\citep{garcia-duran-etal-2018-learning}. ICEWS05-15 is event-heavy, but the gate does not rely on manually labelled static relations. For each relational slot $(h,r)$, an unchanged counterpart provides a low-volatility observation, while a changed counterpart provides a high-volatility observation. We filter non-functional slots whose ratio of unique counterparts to observations exceeds a threshold. The gate MLP is trained with a rotation-based BCE objective:
\begin{equation}
\label{eq:gate-pretrain}
\theta_i = \alpha_{r_i} \cdot \lambda \cdot \Delta t_i, \quad
p_{\mathrm{change}}(\theta_i) = 1 - e^{-\theta_i},
\end{equation}
\begin{equation}
\mathcal{L}_{\mathrm{gate}} = \mathrm{BCE}(y_i,\, p_{\mathrm{change}}(\theta_i))
\end{equation}
where $\Delta t_i$ is the time gap between adjacent observations, $y_i \in \{0,1\}$ indicates entity change. After pretraining, only the MLP weights are retained.

\paragraph{Phase 2: Online Spectrum Learning.}
We load the pretrained gate and keep it frozen. We then learn the global spectrum $(s,\boldsymbol{\omega})$ and the entity and relation embeddings on the target graph using $\mathcal{L}=\mathcal{L}_{\mathrm{triple}}+\lambda_t\mathcal{L}_{\mathrm{time}}+\mathcal{L}_{\mathrm{reg}}$. This phase does not require target-domain volatility labels or fine-tuning the gate, but it does require training on the target graph. The structural loss $\mathcal{L}_{\mathrm{triple}}$ uses 1-vs-all cross-entropy, while the time-contrastive loss $\mathcal{L}_{\mathrm{time}}$ uses a listwise objective:
\begin{align}
\label{eq:time-contrastive-loss}
\mathcal{L}_{\mathrm{time}} &= -\sum_{j=0}^{J} p^{*}_j \log p_j, \nonumber\\
p_j &= \mathrm{softmax}(\mathbf{z})_j, \quad z_0=s(f|\tau), \nonumber\\
z_j &= s(f|\tilde{\tau}_j), \quad j\in\{1,\ldots,J\}
\end{align}
where $p^{*}_j$ uses a Gaussian kernel to softly prefer timestamps close to the validity center. Detailed loss formulations, regularisation, and negative sampling strategies are provided in Appendix~\ref{app:method-details}.

\subsection{Inference-Time Retrieval}
\label{sec:inference}

\paragraph{Geometric Shadowing.}
Training aligns each fact with its observed validity time. At query time $\tau_q$, its phase changes from the observed anchor $t_{\mathrm{happen}}$ by $\Delta\boldsymbol{\theta}_r=s\alpha_r(\tau_q-t_{\mathrm{happen}})\boldsymbol{\omega}$.
For a current-information query ($\tau_q \approx \tau_{\mathrm{now}}$), a recently valid fact is expected to remain close to its learned alignment, whereas an older conflicting fact is more likely to rotate out of phase.
Thus, the new fact naturally shadows the old one without explicit deletion.
Similarly, setting $\tau_q$ to a past date restores the historical fact's alignment while rotating modern facts out of focus.
Appendix~\ref{app:math-analysis-temporal-interpolation} analyses this behaviour in a simplified pairwise setting, and Appendix~\ref{app:qualitative} provides a scoring trace for ICEWS05-15 facts.

\paragraph{Dual-Stream Retrieval.}
We build upon HippoRAG's~\citep{gutierrez2025hipporag2} retrieval pipeline, which computes a semantic score $S_{\mathrm{sem}}$ by combining dense passage similarity with Personalised PageRank over the knowledge graph. We then apply the TKGE scoring function $S_{\mathrm{kge}} = s_{\mathrm{kge}}((h,r,t)\mid\tau)$ from \S\ref{sec:rotation} as a temporal re-ranker. To prevent ``right time, wrong topic'' boosts, we use multiplicative gating with strength $\alpha_g \ge 0$ (set to $0.3$ in all agentic-memory experiments):
$
S_{\mathrm{final}} = S_{\mathrm{sem}} \cdot \left(1 + \alpha_g \cdot S_{\mathrm{kge}}\right)
$,
so temporal signals only amplify facts that are already semantically plausible. We use the same candidate generator as HippoRAG. As a result, \ourapproach cannot recover a fact that is missing from the initial candidate pool, and the direct comparison with HippoRAG isolates the effect of temporal reranking.

\paragraph{Query-Time Modes.}
After the query time and intent have been resolved, we support three retrieval modes:
(1)~\textbf{Explicit Time} ($\tau_q$ present), which scores candidates at a specific timestamp to strictly enforce temporal validity; 
(2)~\textbf{Time-Seeking} (e.g., ``When did X happen?''), which evaluates each candidate against its own stored $t_{\mathrm{happen}}$ to verify internal validity without an external $\tau_q$; 
and (3)~\textbf{Time-Agnostic}, which defaults to $\tau_{\mathrm{now}}$, leveraging geometric shadowing to prioritise fresher dynamic facts.
The temporal extractor handles explicit dates and simple relative expressions. References such as ``then'' must first be resolved from the dialogue history before the resulting time constraint is passed to \ourapproach.

%% file: sections/experiments.tex
\section{Experiments}
\label{sec:experiments}

We evaluate \ourapproach through three research questions: \textbf{(RQ1)} Does the transition from discrete timestamp projections to functional temporal modelling maintain or improve performance on standard TKGE benchmarks? (\S\ref{subsec:rq1_results}); \textbf{(RQ2)} Can \ourapproach outperform existing agentic memory baselines on temporal reasoning tasks while maintaining robustness on non-temporal retrieval? (\S\ref{subsec:rq2_results}); and \textbf{(RQ3)} Can \ourapproach generalise to unseen domain-specific relations? (\S\ref{subsec:rq3_results}).

\subsection{Experimental Setup}
\label{subsec:experimental-setup}

\paragraph{Datasets.}
We evaluate on a diverse set of benchmarks categorised by our research questions.
For RQ1, we use ICEWS05-15 \citep{garcia-duran-etal-2018-learning}.
For RQ2, we evaluate agentic memory across three levels of temporal complexity: (1)~\textbf{Heavy Temporal:} MultiTQ \citep{chen-etal-2023-multi} focuses on temporal reasoning and conflicting facts; (2)~\textbf{Hybrid:} LoCoMo \citep{maharana-etal-2024-evaluating} combines temporal updates with general memory questions; and (3)~\textbf{Largely Non-Temporal:} DMR-MSC \citep{packer2024memgptllmsoperatingsystems} tests whether temporal reranking preserves performance on conversational memory with little temporal conflict.
Finally, for RQ3, we use FinTMMBench \citep{10.1145/3746027.3755723}.
Full dataset statistics are provided in Appendix~\ref{app:datasets}.

\paragraph{Metrics and Baselines.}
For retrieval, we report Mean Reciprocal Rank (MRR), Hits@$k$, and Recall@$k$.
We report answer accuracy (Acc@$k$) using a deterministic cascading verifier for MultiTQ and an LLM judge for DMR-MSC and FinTMMBench. For LoCoMo, Appendix~\ref{app:additional-evaluation} reports the standard partial-match F1 in addition to retrieval recall.
For TKG completion on ICEWS05-15, we compare with non-rotation baselines: DistMult~\citep{DBLP:journals/corr/YangYHGD14a}, DE-SimplE~\citep{Goel_Kazemi_Brubaker_Poupart_2020}, TComplEx~\citep{Lacroix2020Tensor}, TLT-KGE~\citep{10.1145/3511808.3557233-tltkge}, HGE~\citep{Pan_Nayyeri_Li_Staab_2024}, and TimeGate~\citep{11192270}, and rotation-based methods, TeRo~\citep{xu-etal-2020-tero}, ChronoR~\citep{sadeghian2021chronor}, RotateQVS~\citep{chen-etal-2022-rotateqvs}, TeAST~\citep{li-etal-2023-teast}, TCompoundE~\citep{ying-etal-2024-simple}, and 3DG-TE~\citep{li-etal-2025-leveraging-3d}.
For agentic memory benchmarks, we compare against recent graph-based agentic memory systems, including Mem0 \citep{chhikara2025mem0buildingproductionreadyai}, Zep \citep{rasmussen2025zeptemporalknowledgegraph}, LicoMemory \citep{huang2026licomemorylightweightcognitiveagentic} and HippoRAG \citep{gutiérrez2024hipporag,gutierrez2025hipporag2}, as well as a widely used non-graph method, A-Mem~\citep{xu2025amemagenticmemoryllm}.

Detailed metric definitions, answer verification procedures, implementation configurations, and TKGE hyperparameters are provided in Appendices~\ref{app:metrics},~\ref{app:impl-configs}, and~\ref{app:tkge-hyperparams}, respectively.

% Cannot benchmark with MemoTime because it has no triple extractor, not a real agentic memory framework.

% ==================== RQ1 ====================

\begin{table}
\centering
\small
\resizebox{\linewidth}{!}{
\begin{threeparttable}
\caption{Results on ICEWS05-15. Baseline results are taken from \citet{li-etal-2025-leveraging-3d} and \citet{11192270}. Best results are in \textbf{bold}. \colorbox{improve}{Green cells} indicate results where \ourapproach improves over its backbones (DistMult and ChronoR).}
\label{tab:tkge_results}
\begin{tabular}{lcccc}
\toprule
\textbf{Method} & \textbf{MRR} & \textbf{Hit@1} & \textbf{Hit@3} & \textbf{Hit@10} \\
\midrule
\multicolumn{5}{l}{\textbf{Non-Rotation Based}}  \\
\midrule
DistMult  (\citeyear{DBLP:journals/corr/YangYHGD14a})      & 45.6 & 33.7 &  -   & 69.1 \\
DE-SimplE (\citeyear{Goel_Kazemi_Brubaker_Poupart_2020})      & 51.3 & 39.2 & 57.8 & 74.8 \\
TComplEx  (\citeyear{Lacroix2020Tensor})      & 66.5 & 58.3 & 71.6 & 81.1 \\
TLT-KGE   (\citeyear{10.1145/3511808.3557233-tltkge})      & 68.6 & 60.7 & 73.5 & 83.1 \\
HGE       (\citeyear{Pan_Nayyeri_Li_Staab_2024})      & 68.8 & 60.8 & 74.0 & 83.5 \\
TimeGate  (\citeyear{11192270})      & 69.2 & 61.3 & 74.5 & 83.7 \\
\midrule
\multicolumn{5}{l}{\textbf{Rotation Based}}  \\
\midrule
TeRo      (\citeyear{xu-etal-2020-tero})      & 58.6 & 46.9 & 66.8 & 79.5 \\
ChronoR   (\citeyear{sadeghian2021chronor})      & 68.4 & 61.1 & 73.0 & 82.1 \\
RotateQVS (\citeyear{chen-etal-2022-rotateqvs})      & 63.3 & 52.9 & 70.9 & 81.3 \\
TeAST     (\citeyear{li-etal-2023-teast})      & 68.3 & 60.4 & 73.2 & 82.9 \\
TCompoundE (\citeyear{ying-etal-2024-simple})     & 69.2 & 61.2 & 74.3 & 83.7 \\
3DG-TE    (\citeyear{li-etal-2025-leveraging-3d})      & 69.4 & 61.4 & 74.7 & \textbf{84.1} \\
\midrule
\multicolumn{5}{l}{\textbf{\ourapproach (Ours)}} \\
\midrule
\ourapproach-DistMult  & \improved{62.1} & \improved{54.2} & \improved{66.3} & \improved{77.2} \\
\ourapproach-ChronoR   & \improved{\textbf{72.6}} & \improved{\textbf{66.8}} & \improved{\textbf{75.9}} & \improved{83.7} \\
\bottomrule
\end{tabular}
\begin{tablenotes}[flushleft]
\footnotesize
    \item[] We use $k=3$ complex components for ChronoR and \ourapproach-ChronoR, following \citet{sadeghian2021chronor}.
\end{tablenotes}
\end{threeparttable}
}
\vspace{-10pt}
\end{table}

\subsection{Verification of Functional Temporal Modelling (RQ1)}
\label{subsec:rq1_results}

We first test the complete model on temporal knowledge graph completion. This experiment shows whether the functional time operator, Semantic Speed Gate, and time-contrastive loss preserve the link-prediction performance of the DistMult and ChronoR backbones.

As shown in Table~\ref{tab:tkge_results}, \ourapproach-ChronoR achieves an MRR of 72.6, outperforming the vanilla ChronoR (68.4) on ICEWS05-15.
Similarly, our DistMult-based variant, \ourapproach-DistMult, shows a substantial performance improvement (62.1 MRR) compared to the static DistMult baseline (45.6 MRR).
Among the listed methods, \ourapproach-ChronoR obtains the best MRR (72.6), Hit@1 (66.8), and Hit@3 (75.9); its Hit@10 is 83.7 compared with 84.1 for 3DG-TE.

Together, these results show that the complete model improves filtered link prediction over both backbones. We isolate the three components on MultiTQ in Table~\ref{tab:component-ablation}.

% ==================== RQ2 ====================

\subsection{Performance on Episodic and Temporal Memory Tasks (RQ2)}
\label{subsec:rq2_results}

\begin{table*}[t]
\stepcounter{footnote}
\newcommand{\amemfnmark}{\textsuperscript{\thefootnote}}
\centering
\caption{Comprehensive evaluation of \ourapproach. \textbf{(a) MultiTQ:} temporal reasoning. \textbf{(b) LoCoMo:} hybrid retrieval (Recall@10; answer F1 is in Appendix~\ref{app:additional-evaluation}). \textbf{(c) DMR-MSC:} largely non-temporal memory. \textbf{(d) FinTMMBench:} transfer to an unseen domain. Implementation denotes the graph-construction LLM and embedding model. Best results are in \textbf{bold}. \colorbox{improve}{Green cells} mark the direct comparison with HippoRAG under the same graph-construction and semantic-retrieval pipeline.}
\label{tab:main_results_grid}

% --- Row 1, Left: MultiTQ ---
\begin{subtable}[t]{0.5\linewidth}
\centering
\caption{MultiTQ (RQ2, Heavy Temporal)}
\label{tab:multitq_sub}
\resizebox{\linewidth}{!}{
\begin{tabular}{lccccc}
\toprule
\textbf{Method} & \textbf{MRR} & \textbf{Hit@3} & \textbf{Hit@10} & \textbf{Acc@5} & \textbf{Acc@10} \\
\midrule
\multicolumn{6}{l}{\textbf{GPT-5-mini + text-embedding-3-small}} \\
\midrule
Zep & 0.192 & 0.208 & 0.310 & 0.110 & 0.118 \\
Mem0 & 0.174 & 0.190 & 0.282 & 0.122 & 0.122 \\
A-Mem\amemfnmark & -- & -- & -- & -- & --  \\
LicoMem. & 0.149 & 0.160 & 0.292 & 0.114 & 0.128 \\
HippoRAG & 0.203 & 0.232 & 0.348 & 0.112 & 0.102 \\
\ourapproach & \improved{\textbf{0.337}} & \improved{\textbf{0.384}} & \improved{\textbf{0.502}} & \improved{\textbf{0.366}} & \improved{\textbf{0.392}} \\
\midrule
\multicolumn{6}{l}{\textbf{LLaMA-3.1-70B + BGE-M3}} \\
\midrule
Zep & 0.217 & 0.252 & 0.370 & 0.098 & 0.116 \\
Mem0 & 0.228 & 0.264 & 0.356 & 0.120 & 0.114 \\
A-Mem\amemfnmark & -- & -- & -- & -- & -- \\
LicoMem. & 0.159 & 0.182 & 0.304 & 0.114 & 0.120 \\
HippoRAG & 0.236 & 0.266 & 0.354 & 0.122 & 0.116 \\
\ourapproach & \improved{\textbf{0.316}} & \improved{\textbf{0.342}} & \improved{\textbf{0.440}} & \improved{\textbf{0.312}} & \improved{\textbf{0.338}} \\
\bottomrule
\end{tabular}
}
\end{subtable}
\hfill
% --- Row 1, Right: LoCoMo ---
\begin{subtable}[t]{0.49\linewidth}
\centering
\caption{LoCoMo (RQ2, Hybrid Tasks)}
\label{tab:locomo_sub}
\resizebox{\linewidth}{!}{
\begin{tabular}{lccccc}
\toprule
\textbf{Method} & \textbf{\makecell{Single\\Hop}} & \textbf{\makecell{Multi\\Hop}} & \textbf{\makecell{Open\\Domain}} & \textbf{\makecell{Temporal\\Reason}} & \textbf{Average} \\
\midrule
\multicolumn{6}{l}{\textbf{GPT-5-mini + text-embedding-3-small}} \\
\midrule
Zep    & 0.557 & \textbf{0.861} & 0.831 & 0.553 & 0.770 \\
Mem0   & 0.740 & 0.832 & 0.883 & 0.690 & 0.834 \\
A-Mem  & 0.740 & 0.846 & 0.860 & 0.691 & 0.825  \\
LicoMem. & 0.727 & 0.856 & 0.848 & 0.661 & 0.816  \\
HippoRAG & 0.711 & 0.837 & 0.862 & 0.645 & 0.815  \\
\ourapproach & \improved{\textbf{0.768}} & \improved{0.850} & \improved{\textbf{0.904}} & \improved{\textbf{0.726}} & \improved{\textbf{0.857}} \\

\midrule
\multicolumn{6}{l}{\textbf{LLaMA-3.1-70B + BGE-M3}} \\
\midrule
Zep    & 0.557 & \textbf{0.861} & 0.831 & 0.553 & 0.770 \\
Mem0   & 0.746 & 0.860 & 0.875 & 0.737 & \textbf{0.839} \\
A-Mem  & 0.658 & 0.776 & 0.777 & 0.702 & 0.750  \\
LicoMem.  & 0.605 & 0.768 & 0.725 & 0.584 & 0.703  \\
HippoRAG  & 0.717 & 0.852 & 0.870 & 0.732 & 0.830  \\
\ourapproach & \improved{\textbf{0.759}} & 0.824 & \improved{\textbf{0.879}} & \improved{\textbf{0.759}} & \improved{0.838} \\
\bottomrule
\end{tabular}
}
\end{subtable}

\vspace{10pt}

% --- Row 2, Left: DMR-MSC ---
\begin{subtable}[t]{0.5\linewidth}
\centering
\caption{DMR-MSC (RQ2, Largely Non-Temporal)}
\label{tab:dmrmsc_sub}
\resizebox{\linewidth}{!}{
\begin{tabular}{lccccc}
\toprule
\textbf{Method} & \textbf{MRR} & \textbf{Hit@1} & \textbf{Hit@3} & \textbf{Acc@5} & \textbf{Acc@10} \\
\midrule
\multicolumn{6}{l}{\textbf{GPT-5-mini + text-embedding-3-small}} \\
\midrule
Zep    & 0.170 & 0.110 & 0.180 & 0.302 & 0.376 \\
Mem0   & 0.847 & 0.766 & 0.926 & 0.858 & 0.848 \\
A-Mem  & 0.825 & 0.732 & 0.912 & 0.848 & 0.856  \\
LicoMem. & 0.326 & 0.224 & 0.372 & 0.670 & 0.728  \\
HippoRAG & 0.848 & 0.768 & 0.926 & 0.852 & 0.850 \\
\ourapproach & \improved{\textbf{0.856}} & \improved{\textbf{0.774}} & \improved{\textbf{0.934}} & \improved{\textbf{0.862}} & \improved{\textbf{0.858}} \\
\midrule
\multicolumn{6}{l}{\textbf{LLaMA-3.1-70B + BGE-M3}} \\
\midrule
Zep   & 0.333 & 0.232 & 0.394 & 0.384 & 0.428 \\
Mem0  & 0.821 & 0.714 & 0.926 & 0.758 & 0.770 \\
A-Mem & 0.823 & 0.732 & 0.902 & 0.728 & 0.738  \\
LicoMem. & 0.202 & 0.138 & 0.228 & 0.258 & 0.338  \\
HippoRAG  & 0.818 & 0.718 & 0.912 & 0.768 & 0.776 \\
\ourapproach                                        & \improved{\textbf{0.847}} & \improved{\textbf{0.760}} & \improved{\textbf{0.930}} & \improved{\textbf{0.774}} & \improved{\textbf{0.786}} \\
\bottomrule
\end{tabular}
}
\end{subtable}
\hfill
% --- Row 2, Right: FinTMMBench ---
\begin{subtable}[t]{0.49\linewidth}
\centering
\caption{FinTMMBench (RQ3)}
\label{tab:fintmmbench_sub}
\resizebox{\linewidth}{!}{
\begin{tabular}{lccccc}
\toprule
\textbf{Method} & \textbf{MRR} & \textbf{R@5} & \textbf{R@10} & \textbf{Acc@5} & \textbf{Acc@10} \\
\midrule
\multicolumn{6}{l}{\textbf{GPT-5-mini + text-embedding-3-small}} \\
\midrule
Zep   & 0.703 & 0.644 & 0.759 & 0.480 & 0.520 \\
Mem0  & 0.691 & 0.645 & 0.768 & 0.550 & 0.610 \\
A-Mem  & 0.716 & 0.647 & 0.796 & 0.540 & 0.640  \\
LicoMem. & 0.488 & 0.480 & 0.609 & 0.480 & 0.590 \\
HippoRAG  & 0.690 & 0.645 & 0.768 & 0.550 & \textbf{0.650} \\
\ourapproach & \improved{\textbf{0.728}} & \improved{\textbf{0.673}} & \improved{\textbf{0.779}} & \improved{\textbf{0.580}} & \textbf{0.650} \\
\midrule
\multicolumn{6}{l}{\textbf{LLaMA-3.1-70B + BGE-M3}} \\
\midrule
Zep  & 0.515 & 0.510 & 0.591 & 0.430 & 0.450 \\
Mem0 & 0.718 & 0.647 & 0.765 & 0.570 & 0.610 \\
A-Mem  & 0.650 & 0.631 & 0.742 & 0.520 & 0.590  \\
LicoMem. & 0.554 & 0.559 & 0.662 & 0.460 & 0.520 \\
HippoRAG  & 0.724 & 0.680 & 0.766 & 0.610 & 0.610 \\
\ourapproach  & \improved{\textbf{0.726}} & \improved{\textbf{0.707}} & \improved{\textbf{0.793}} & \improved{\textbf{0.620}} & \improved{\textbf{0.650}} \\
\bottomrule
\end{tabular}
}
\end{subtable}
\end{table*}
% \footnotetext[\value{footnote}]{A-Mem is a non-graph agentic memory that does not natively support structured triple ingestion. It is therefore excluded from MultiTQ, which requires incremental ingestion of ${\sim}$11K temporal triples.}
\footnotetext[\value{footnote}]{A-Mem is excluded from MultiTQ because it lacks native support for massive structured triple ingestion (${\sim}$11K).}

To answer RQ2, we evaluate whether \ourapproach can resolve complex temporal conflicts in agentic memory without degrading foundational, non-temporal retrieval capabilities.
We evaluate all agentic memory methods under two implementation configurations: a closed-source API setup (GPT-5-mini with text-embedding-3-small) and an open-source setup (LLaMA-3.1-70B with BGE-M3) to serve as a robustness check.
Because \ourapproach adds only temporal reranking to HippoRAG, the direct comparison between the two systems isolates the effect of our method. Comparisons with Mem0, Zep, A-Mem, and LicoMemory provide broader system-level context.

\paragraph{Temporal Reasoning (MultiTQ).}
MultiTQ directly tests temporal reasoning over changing and conflicting facts. With GPT-5-mini for graph construction, adding \ourapproach to HippoRAG improves MRR from 0.203 to 0.337 and Acc@5 from 0.112 to 0.366 (Table~\ref{tab:multitq_sub}). The open-source configuration shows similar gains: MRR increases from 0.236 to 0.316 and Acc@5 from 0.122 to 0.312. The improvement therefore holds across both graph-construction and embedding configurations.

Table~\ref{tab:component-ablation} serves as an ablation on MultiTQ.
Removing the Semantic Speed Gate lowers MRR from 0.3371 to 0.1044, showing that one shared rotation speed is unsuitable for both persistent and evolving relations. 
Replacing functional time with learned timestamp embeddings lowers MRR further to 0.0452. 
Removing the time-contrastive loss has a smaller effect: MRR decreases from 0.3371 to 0.3254 and Acc@5 from 0.366 to 0.340. 
Functional time and semantic gating therefore provide the main gains, while the time-contrastive loss provides an additional improvement.

\begin{table}[t]
\centering
\footnotesize
\setlength{\tabcolsep}{3pt}
\caption{Single-factor ablation on the fixed MultiTQ subset (selected metrics; full results in Appendix~\ref{app:component-ablation-details}). ``No gate'' uses one speed for all relations; ``discrete time'' uses learned timestamp lookups.}
\label{tab:component-ablation}
\begin{tabular}{lrrrr}
\toprule
\textbf{Variant} & \textbf{MRR} & \textbf{H@3} & \textbf{Acc@5} & \textbf{Acc@10} \\
\midrule
Full \ourapproach & \textbf{0.3371} & \textbf{0.384} & \textbf{0.366} & \textbf{0.392} \\
w/o time contrastive & 0.3254 & 0.368 & 0.340 & 0.390 \\
w/o semantic gate & 0.1044 & 0.116 & 0.100 & 0.108 \\
Discrete time & 0.0452 & 0.046 & 0.040 & 0.040 \\
\bottomrule
\end{tabular}
\end{table}

\paragraph{Hybrid Retrieval (LoCoMo).}
LoCoMo contains single-hop, multi-hop, open-domain, and temporal questions. \ourapproach improves the average Recall@10 of HippoRAG from 0.815 to 0.857 with GPT-5-mini and from 0.830 to 0.838 with the open-source configuration (Table~\ref{tab:locomo_sub}). In the GPT-5-mini configuration, the largest gain is on temporal questions, where Recall@10 increases from 0.645 to 0.726. These results measure retrieval quality. Appendix~\ref{app:additional-evaluation} reports answer-level partial-match F1, where \ourapproach matches HippoRAG at 0.457 in the GPT-5-mini configuration and improves by 0.001 in the open-source configuration. Thus, the retrieval gains do not lead to a comparable improvement in answer F1 with the current reader.

\paragraph{Preservation of General Memory (DMR-MSC).}
DMR-MSC tests conversational memory in a setting with little explicit temporal conflict.
As shown in Table~\ref{tab:dmrmsc_sub}, \ourapproach achieves an MRR of 0.856 and an LLM@5 Accuracy of 0.862 under the GPT-5-mini setup, slightly improving upon the baseline HippoRAG performance of 0.848 and 0.852, respectively.
We observe similar gains in the LLaMA-3.1-70B implementation.
These results show that the temporal reranker preserves performance on this largely non-temporal benchmark. The ranking of the methods is unchanged when we use the stricter conflict-aware judge described in Appendix~\ref{app:additional-evaluation}.

\paragraph{Conclusion for RQ2.}
Across both configurations, \ourapproach improves temporal retrieval over the same HippoRAG candidate pool and preserves performance when temporal conflicts are largely absent.

% ==================== RQ3 ====================

\subsection{Domain Generalisation (RQ3)}
\label{subsec:rq3_results}

To answer RQ3, we evaluate \ourapproach on FinTMMBench to test zero-shot generalisation in high-volatility financial contexts \citep{10.1145/3768623,li2026llmbasedfinancialinvestingstrategies}.
As shown in Table~\ref{tab:fintmmbench_sub}, \ourapproach achieves 0.728 MRR and 0.580 Acc@5 with GPT-5-mini, compared with 0.690 and 0.550 for HippoRAG. In the open-source configuration, \ourapproach improves retrieval recall, while MRR increases slightly from 0.724 to 0.726.
The frozen gate therefore transfers from geopolitical event relations to the unseen financial relations evaluated here. \ourapproach also remains the top-ranked method under the stricter conflict-aware judge (Appendix~\ref{app:additional-evaluation}).

%% file: sections/conclusion.tex
\section{Conclusion}
\label{sec:conclusion}

We identified two limitations in how graph-based memory systems handle time. First, discrete timestamp metadata applies the same policy to static and dynamic relations, so recency sorting can bury static facts. Second, existing conflict-management strategies either overwrite history or require an LLM at every ingestion step, which becomes costly at scale.
\ourapproach addresses these limitations by representing time as a continuous phase rotation in the KG embedding space. Its pretrained Semantic Speed Gate maps relation text embeddings to rotation speeds, keeping persistent relations aligned while rotating outdated facts out of phase. Because old facts remain stored, the memory is append-only and can answer both current and historical queries.
Empirically, \ourapproach achieves state-of-the-art TKGE results and large gains over the corresponding HippoRAG pipeline on temporal memory, while preserving performance on largely non-temporal memory and transferring the frozen gate to an unseen financial domain.

%% file: sections/limitation.tex
\section*{Limitations}

Our work has three main limitations. First, the Semantic Speed Gate is pretrained on ICEWS05-15, a political events dataset. While it generalises zero-shot to unseen domains in our experiments, incorporating additional domain-specific temporal data during pretraining, which is efficient due to the lightweight MLP architecture, could further improve calibration for specialised settings.

Second, our evaluation is conducted exclusively on English-language benchmarks. The effectiveness of the text-embedding-based gate on multilingual or low-resource settings remains to be validated.

Additionally, the current system operates on textual relational facts and does not handle multimodal inputs such as images, or audio, which are increasingly relevant in real-world agentic applications. Extending \ourapproach to multimodal temporal reasoning is a promising direction for future work.

%% file: sections/acknowledgements.tex
\section*{Acknowledgements}

AI assistants were used only to polish author-written text and to help debug code. All suggestions were reviewed and validated by the authors, who take full responsibility for the paper and code.

%% file: appendices/appendix-master.tex
\newcounter{promptbox}
\renewcommand{\thepromptbox}{\Roman{promptbox}}

\newcommand{\stepbox}[1]{\refstepcounter{promptbox}\label{#1}}
\newtcolorbox{promptbox}[1]{%
  colback=black!4,
  colframe=black!55,
  coltitle=white,
  fonttitle=\bfseries\small,
  title={Box~\thepromptbox: #1},
  boxrule=0.6pt,
  arc=2pt,
  left=6pt, right=6pt, top=4pt, bottom=4pt,
  breakable
}

\input{appendices/extended-methodology}  
\input{appendices/prompts}              
\input{appendices/datasets}             
\input{appendices/licenses}
\input{appendices/metrics}             
\input{appendices/additional-evaluation}
\input{appendices/extended-experimental-setup}   
\input{appendices/math-analysis}
\input{appendices/qualitative-examples}

%% file: appendices/extended-methodology.tex
\section{Scoring, Training, and Gate Formulations}
\label{app:method-details}

\subsection{Scoring Variants}
\label{app:unrotation}

\paragraph{Efficient 1-vs-$N$ Retrieval.}
In practice, the trace-diagonal sum across $k$ components reduces to a flat element-wise product and summation over the $k \times 2d$ real dimensions. An important consequence of this structure is the \emph{unrotation trick}: to score a query $(h,r,?)$ against all $N$ entities simultaneously, we form $\mathbf{q} = \mathrm{Rot}(\mathbf{e}_h, \boldsymbol{\theta}) \odot \mathbf{w}_r \odot \hat{\mathbf{w}}_r$ and then \emph{unrotate} by $-\boldsymbol{\theta}$, yielding a vector in the same space as the raw entity embeddings. A single matrix multiplication $\mathrm{Rot}(\mathbf{q}, -\boldsymbol{\theta}) \cdot \mathbf{E}^{\top}$ then produces scores for all entities without materialising $N$ separate rotations, enabling efficient 1-vs-all training.
A formal proof with complexity analysis is provided in \S\ref{app:math-analysis}.

\paragraph{DistMult Variant.}
\label{app:distmult}
As a simplified variant, setting $k{=}1$ and removing the inverse relation table recovers a time-conditioned DistMult~\citep{DBLP:journals/corr/YangYHGD14a} backbone. This variant uses self-adversarial negative sampling instead of the 1-vs-all cross-entropy loss: $\mathcal{L}_{\mathrm{triple}} = -\log\sigma(s^{+}) - \sum_{k} w_k \log\sigma(-s^{-}_k)$, where $w_k = \mathrm{softmax}(s^{-}_k / T)$ are self-adversarial weights~\citep{DBLP:conf/iclr/SunDNT19}. For regularisation, we use global L3 on the full embedding tables instead of per-batch N3.

\subsection{Training Objectives}
\label{app:triple-loss}

\paragraph{Structural Triple Loss.}
With the ChronoR backbone, we use a 1-vs-all cross-entropy loss exploiting the unrotation trick (\S\ref{app:unrotation}). The unrotated query $\mathbf{q}_t(\tau) = \mathrm{Rot}\bigl(\mathrm{Rot}(\mathbf{e}_h, \boldsymbol{\theta}_r(\tau)) \odot \mathbf{w}_r \odot \hat{\mathbf{w}}_r,\, -\boldsymbol{\theta}_r(\tau)\bigr)$ lives in the same space as the raw entity embeddings, enabling a single matrix multiplication $\mathbf{q}_t(\tau) \cdot \mathbf{E}^{\top}$ to score all entities. Crucially, $\mathbf{q}_t(\tau)$ remains $\tau$-dependent because the relation weights break the symmetry between the forward and inverse rotations; consequently, the same $(h,r)$ pair produces distinct query vectors at different timestamps, allowing the structural loss alone to learn temporal discrimination:
\begin{equation}
\begin{split}
\mathcal{L}_{\mathrm{triple}} = \tfrac{1}{2}\bigl(
&\mathrm{CE}_{\mathrm{tail}}(\mathbf{q}_t(\tau) \cdot \mathbf{E}^{\top}, t)\\
+\,&\mathrm{CE}_{\mathrm{head}}(\mathbf{q}_h(\tau) \cdot \mathbf{E}^{\top}, h)\bigr)
\end{split}
\end{equation}
where $\mathrm{CE}$ is the standard cross-entropy over the full entity vocabulary and $\mathbf{q}_h(\tau)$ is the symmetric head-prediction query. This loss is computed with the gate $\alpha_r$ \emph{detached}, so gradients update only the entity, relation, and time parameters.

\paragraph{Conflict-Aware Negative Sampling.}
\label{app:conflict-neg}
We enhance both training variants with \textit{conflict-aware negative sampling}: rather than sampling random entities, we prioritise sampling competing tails from the same $(h,r)$ group when available. This forces the model to discriminate between mutually exclusive facts (e.g., Obama born in \textit{Hawaii} vs. \textit{Kenya}) rather than easy negatives.

\paragraph{Regularisation.}
\label{app:regularisation}
We apply backbone-specific embedding regularisation.
For ChronoR, we use per-batch N3 regularisation (fourth-power penalty): $\mathcal{L}_{\mathrm{reg}} = \lambda_r \bigl(\|\mathbf{e}_h\|_4^4 + \|\mathbf{w}_r\|_4^4 + \|\hat{\mathbf{w}}_r\|_4^4 + \|\mathbf{e}_t\|_4^4\bigr) / B$, where $B$ is the batch size.
For the DistMult variant, we use global L3 regularisation on the full embedding tables.
Additionally, a gate regularisation term $\mathcal{L}_{\mathrm{gate\_reg}} = \overline{\alpha}_r^{\,\text{non-comp}}$ softly encourages $\alpha_r \to 0$ for non-competing relation slots, reinforcing the temporal clutch's static behaviour where no temporal discrimination is needed.

\paragraph{Time-Contrastive Loss.}
\label{app:time-loss}
The listwise time-contrastive loss $\mathcal{L}_{\mathrm{time}}$ (Equation~\ref{eq:time-contrastive-loss} in \S\ref{sec:training}) uses a Gaussian target kernel whose width $\sigma$ follows a cosine curriculum from $\sigma_{\text{start}} = 0.5$\,yr to $\sigma_{\text{end}} = 0.02$\,yr over the configured decay epochs (Table~\ref{tab:tkge-hyperparams}), progressively sharpening temporal discrimination.
Negative times are sampled preferentially from the same $(h, r)$ slot history when available, with jitter ($\pm 0.02$ years) and one forced far negative ($\pm 365$ days).
A minimum-gap curriculum decays from 90 to 3 days over 60 epochs to avoid trivially easy negatives early in training.

\subsection{Semantic Speed Gate Pretraining}
\label{app:gate-pretrain}

The semantic speed gate $\alpha_r = \text{MLP}(\phi(r))$ is pretrained on self-supervised transition observations mined from ICEWS05-15 before online TKGE training begins.

\paragraph{Stage 1: Transition Mining.}
We group all temporal triples by $(h, r)$ slot and identify consecutive temporal observations.
For each pair of adjacent observations at times $t_i$ and $t_{i+1}$, we record a binary \texttt{changed} label indicating whether the tail entity changed.
Filtering criteria: minimum 3 observations per slot, maximum functional ratio 0.5, and up to 256 pairs per slot.

\paragraph{Stage 2: Gate Training.}
The gate MLP is trained with the rotation-based BCE objective defined in Equation~\ref{eq:gate-pretrain} (\S\ref{sec:training}), with $p_{\text{change}}$ clamped to avoid numerical overflow.
Training uses 100 epochs with learning rate $5 \times 10^{-4}$, embedding batch size 64, and class-weighted loss with auto-computed weights.
The pretrained checkpoint is loaded and frozen at the start of online TKGE training.

%% file: appendices/prompts.tex
% ---------------------------------------------------------------
%  Appendix: Prompts
% ---------------------------------------------------------------

\section{Prompts}
\label{app:prompts}

\subsection{Answer Generation}
\label{app:answer-gen}

\paragraph{Answer LLM.}
All downstream answer generation uses GPT-5.2 with \texttt{temperature=0} and JSON response format $\{$\texttt{"answer": "<text>"}$\}$.
The answer LLM is always routed to the OpenAI API, independent of whether the graph construction LLM is hosted locally via vLLM.

\paragraph{System prompts.}
The system prompt varies by benchmark:
\begin{itemize}[leftmargin=2em]
    \item \textbf{MultiTQ, DMR-MSC}: ``You answer questions using only provided facts.''
    \item \textbf{LoCoMo}: ``You answer using only the provided context.''
    \item \textbf{FinTMMBench}: ``You are a financial analyst assistant. Use only the provided financial data to answer the question. Be concise and precise.''
\end{itemize}

Box~\ref{box:answer-prompt-multitq} and Box~\ref{box:answer-prompt-fintmm} show the user prompt templates for MultiTQ and FinTMMBench, respectively.

\stepbox{box:answer-prompt-multitq}
\begin{promptbox}{Answer Generation Prompt (MultiTQ / DMR-MSC)}
\begin{lstlisting}[
basicstyle=\ttfamily\small,
breaklines=true,
breakindent=0pt,
columns=fullflexible,
keepspaces=true,
showstringspaces=false
]
Question: {question}

Supporting facts:
{bullet_list_of_retrieved_documents}

Provide a concise answer grounded in the facts.
Respond with JSON {"answer": "<text>"}.
\end{lstlisting}
\end{promptbox}

\stepbox{box:answer-prompt-fintmm}
\begin{promptbox}{Answer Generation Prompt (FinTMMBench)}
\begin{lstlisting}[
basicstyle=\ttfamily\small,
breaklines=true,
breakindent=0pt,
columns=fullflexible,
keepspaces=true,
showstringspaces=false
]
Question: {question}

Financial data:
{bullet_list_of_retrieved_documents}

Provide a short answer grounded in the data.
Respond with JSON {"answer": "<text>"}.
\end{lstlisting}
\end{promptbox}

\paragraph{LoCoMo answer prompts.}
For LoCoMo, we follow the original benchmark protocol~\citep{maharana-etal-2024-evaluating}:
\begin{itemize}[leftmargin=2em]
    \item \textbf{Categories 1-4} (open-ended): \texttt{"Based on the above context, write an answer in the form of a short phrase for the following question. Answer with exact words from the context whenever possible. Question: \{q\} Short answer:"}
    \item \textbf{Category 5} (adversarial): \texttt{"Based on the above context, answer the following question. Question: \{q\} Short answer:"}
    \item \textbf{Multiple choice}: \texttt{"Based on the above context, choose the best answer from the options below. Respond with the exact choice text. Question: \{q\} Choices: \{choices\} Answer:"}
\end{itemize}
All LoCoMo responses are appended with: \texttt{Respond with JSON \{"answer": "<text>"\}.}

% ---------------------------------------------------------------
\subsection{Information Extraction}
\label{app:ie-prompts}

\ourapproach uses three LLM-based extraction stages during knowledge graph construction: named entity recognition (Appendix~\ref{app:ner-prompt}), temporal triple extraction (Appendix~\ref{app:triple-prompt}), and query-time entity and time extraction (Appendix~\ref{app:query-ner-prompt}, Appendix~\ref{app:time-extraction-prompt}).

\subsubsection{Named Entity Recognition (NER)}
\label{app:ner-prompt}

Box~\ref{box:ner-prompt} shows the one-shot NER prompt used to extract entities from each passage during graph construction.

\stepbox{box:ner-prompt}
\begin{promptbox}{Named Entity Recognition Prompt (One-Shot)}
\textbf{System:}
\begin{lstlisting}[
basicstyle=\ttfamily\small,
breaklines=true,
breakindent=0pt,
columns=fullflexible,
keepspaces=true,
showstringspaces=false
]
Your task is to extract named entities from the given
paragraph. Respond with a JSON list of entities.

Rules:
- Only include real-world entities, people, organizations, locations, products, and named events.
- Do not include time expressions, dates, durations, or clock times (these are handled separately).
- Do not include numbers that are only quantities or durations (e.g., "5 years").

One-shot example:

User: "Radio City is India's first private FM radio station and was started on 3 July 2001\ldots"

Assistant: {"named_entities": ["Radio City", "India", "Hindi", "English", "PlanetRadiocity.com"]}
\end{lstlisting}
\end{promptbox}

\subsubsection{Temporal Triple Extraction (OpenIE)}
\label{app:triple-prompt}

Box~\ref{box:triple-prompt} shows the system prompt for NER-conditioned triple extraction with temporal metadata.
This prompt is paired with three few-shot examples covering standard extraction, relative time resolution, and duration inference.

\stepbox{box:triple-prompt}
\begin{promptbox}{Temporal Triple Extraction System Prompt}
\begin{lstlisting}[
basicstyle=\ttfamily\small,
breaklines=true,
breakindent=0pt,
columns=fullflexible,
keepspaces=true,
showstringspaces=false
]
Your task is to construct an RDF graph from the given passages and named entity lists.
Respond with JSON where each triple entry also carries timing metadata.

Requirements:
- Represent each triple as an object with fields: head, relation, tail, text_time, observed_time.
- text_time must be a normalized date string in YYYY-MM-DD, or "" if no time is mentioned.
- If only a month or year is mentioned, use the first day of that month or year.
- If the year is missing, use the year from observed_time.
- Resolve relative expressions using the closest explicit date in the passage. If none exists, use observed_time.
- observed_time is always set to the provided observed_time string.
- Do not infer a time if the passage does not contain a temporal expression.
- Do NOT use time expressions as head or tail entities; keep time only in text_time.
- Each triple should contain at least one, but preferably two, of the named entities in the list.
- Resolve pronouns to their specific names.
- Extract all factual relations, including actions, states, plans, and events.
- Do not omit facts; prioritize completeness over brevity.
- Duration handling: if the passage states a duration such as "for 5 years" and an explicit reference date is present, infer the start date and put it in text_time.
\end{lstlisting}
\vspace{2pt}
\textbf{Output format:}
\begin{lstlisting}[
basicstyle=\ttfamily\small,
breaklines=true,
breakindent=0pt,
columns=fullflexible,
keepspaces=true,
showstringspaces=false
]
{"triples": [
  {"head": "X", "relation": "Y", "tail": "Z",
   "text_time": "YYYY-MM-DD",
   "observed_time": "2024-01-01T00:00:00Z"}
]}
\end{lstlisting}
\end{promptbox}

\subsubsection{Query-Time Entity Extraction}
\label{app:query-ner-prompt}

At query time, we extract named entities from the question to initialise graph traversal.

\stepbox{box:query-ner-prompt}
\begin{promptbox}{Query NER Prompt (One-Shot)}
\textbf{System:} 
\begin{lstlisting}[
basicstyle=\ttfamily\small,
breaklines=true,
breakindent=0pt,
columns=fullflexible,
keepspaces=true,
showstringspaces=false
]
"You're a very effective entity extraction system."
\end{lstlisting}
\textbf{User:} 
\begin{lstlisting}[
basicstyle=\ttfamily\small,
breaklines=true,
breakindent=0pt,
columns=fullflexible,
keepspaces=true,
showstringspaces=false
]
Please extract all named entities that are important for solving the questions below. Place the named entities in json format.
Question: Which magazine was started first Arthur's Magazine or First for Women?
\end{lstlisting}
\textbf{Assistant:}
\begin{lstlisting}[
basicstyle=\ttfamily\small,
breaklines=true,
breakindent=0pt,
columns=fullflexible,
keepspaces=true,
showstringspaces=false
]
{"named_entities": ["First for Women", "Arthur's Magazine"]}
\end{lstlisting}
\end{promptbox}

\subsubsection{Time Extraction}
\label{app:time-extraction-prompt}

At query time, we extract temporal constraints and ordering intent from the question.

\stepbox{box:time-extraction-prompt}
\begin{promptbox}{Time Extraction Prompt}
\textbf{System:}
\begin{lstlisting}[
basicstyle=\ttfamily\small,
breaklines=true,
breakindent=0pt,
columns=fullflexible,
keepspaces=true,
showstringspaces=false
]
"You extract the time constraint from a query. Return a single line only."
\end{lstlisting}
\textbf{User:}
\begin{lstlisting}[
basicstyle=\ttfamily\small,
breaklines=true,
breakindent=0pt,
columns=fullflexible,
keepspaces=true,
showstringspaces=false
]
Given the query, output the time constraint, temporal
ordering intent, and whether the query asks for a time.
Return exactly one line in this format:
time=YYYY-MM-DD; ordering=earliest|latest|none;
time_request=yes|no

Rules:
- If a time constraint exists, return it as YYYY-MM-DD.
- If the query only specifies a month or year, use the first day of that month or year.
- If the query omits the year, use the year from the reference date.
- Resolve relative expressions using the reference date.
- If no time constraint exists, return time=NONE.
- ordering=earliest for queries like "earliest/first/oldest".
- ordering=latest for queries like "latest/most recent/last time".
- ordering=none otherwise.
- time_request=yes if the query asks for a time as the answer (e.g., when/what year/which year/what date).
- time_request=no if the query only uses time as a constraint or does not ask for time.

Reference date (UTC): {reference_date}
Query: {query}
\end{lstlisting}
\end{promptbox}

%% file: appendices/datasets.tex
% ---------------------------------------------------------------
%  Appendix: Benchmark Datasets
% ---------------------------------------------------------------

\section{Benchmark Datasets}
\label{app:datasets}

Table~\ref{tab:dataset-stats} summarises the key statistics for each dataset.
Detailed descriptions follow.

\begin{table*}[h]
\centering
\small
\caption{Summary statistics of evaluation datasets.}
\label{tab:dataset-stats}
\begin{tabular}{llrrp{3.2cm}}
\toprule
\textbf{Dataset} & \textbf{RQ} & \textbf{Facts/Docs} & \textbf{Queries} & \textbf{Task Type} \\
\midrule
ICEWS05-15  & RQ1 & 461{,}329  & 46{,}092      & TKG completion \\
MultiTQ     & RQ2 & 11{,}074  & 500     & Temporal KGQA \\
LoCoMo      & RQ2 & 10 conv.   & 1{,}986 & Conv.\ memory QA \\
DMR-MSC     & RQ2 & 500 dial.  & 500     & Dialogue memory QA \\
FinTMMBench & RQ3 & 908 docs$^\dagger$  & 100$^\dagger$ & Financial temporal QA \\ 
\bottomrule
\end{tabular}
\vspace{2pt}

\raggedright{\footnotesize $^\dagger$Stratified sample: 100 questions with a corpus of 908 documents (227 gold sources + 681 randomly sampled non-gold documents).} 
\end{table*}

\paragraph{ICEWS05-15~\citep{garcia-duran-etal-2018-learning}.}
The Integrated Crisis Early Warning System (ICEWS) dataset contains geopolitical event triples spanning from 2005 to 2015.
Each triple takes the form \texttt{(head, relation, tail, date)} with dates in \texttt{YYYY-MM-DD} format.
The standard split comprises 368{,}962 training, 46{,}275 validation, and 46{,}092 test triples.
We use this dataset exclusively for RQ1 to verify that our functional temporal modelling preserves standard TKGE accuracy.

\paragraph{MultiTQ~\citep{chen-etal-2023-multi}.}
The full benchmark builds on the ICEWS05-15 KG, which contains 461{,}329 temporal quadruples at 4{,}017 timestamps. Its test set contains 54{,}584 questions at day, month, and year resolution.
Our end-to-end agentic-memory evaluation uses a fixed set of questions: the first 500 entries in the released \texttt{test.json} (IDs 3000000-3000499). We selected this set before running any method and did not filter questions based on their results. We process only the corresponding time snapshots, yielding 11{,}074 facts ingested incrementally per snapshot.
Evaluating all 54{,}584 questions across ten method/configuration combinations with both top-5 and top-10 contexts would require at least 1{,}091{,}680 answer-generation calls, excluding framework-specific graph-construction calls. We therefore use the fixed 500-question set for the end-to-end evaluation.
As shown in Table~\ref{tab:multitq-subset-distribution}, the question- and answer-type distributions differ from those of the full test set by at most 2.4 percentage points.

\begin{table}[h]
\centering
\small
\caption{Distribution of the full MultiTQ test set and the fixed 500-question evaluation subset.}
\label{tab:multitq-subset-distribution}
\begin{tabular}{lrr}
\toprule
\textbf{Category} & \textbf{Full test set} & \textbf{Fixed subset} \\
\midrule
\texttt{equal} & 17{,}311 (31.71\%) & 155 (31.00\%) \\
\texttt{before\_after} & 11{,}073 (20.29\%) & 102 (20.40\%) \\
\texttt{first\_last} & 10{,}480 (19.20\%) & 108 (21.60\%) \\
\texttt{equal\_multi} & 3{,}207 (5.88\%) & 23 (4.60\%) \\
\texttt{after\_first} & 6{,}266 (11.48\%) & 49 (9.80\%) \\
\texttt{before\_last} & 6{,}247 (11.44\%) & 63 (12.60\%) \\
\midrule
Entity answer & 38{,}700 (70.90\%) & 345 (69.00\%) \\
Time answer & 15{,}884 (29.10\%) & 155 (31.00\%) \\
\bottomrule
\end{tabular}
\end{table}

\paragraph{LoCoMo~\citep{maharana-etal-2024-evaluating}.}
Long-Context Conversational Memory benchmark consisting of 10 synthetic multi-session conversations with 1{,}986 question-answer pairs.
Questions are divided into five categories:
(1)~\textit{Single-Hop}: direct fact retrieval;
(2)~\textit{Multi-Hop}: multi-step reasoning;
(3)~\textit{Temporal Reasoning}: time-sensitive queries requiring date resolution;
(4)~\textit{Open Domain}: general knowledge queries; and
(5)~\textit{Adversarial}: queries about information not present in the conversations.
Each question is accompanied by evidence references (e.g., \texttt{D1:3}) pointing to specific conversation segments.

\paragraph{DMR-MSC~\citep{packer2024memgptllmsoperatingsystems}.}
The Dynamic Memory Retrieval Multi-Session Chat dataset contains 500 multi-session dialogues with self-instructed question-answer pairs.
Each example includes persona statements, dialog turns with speaker identities, and temporal context via \texttt{time\_back} annotations (e.g., ``14 days'').
We use this largely non-temporal conversational benchmark to check that temporal reranking does not reduce standard memory retrieval performance.

\paragraph{FinTMMBench~\citep{10.1145/3746027.3755723}.}
Financial Temporal Multi-Modal Benchmark containing 5{,}676 question-answer pairs over NASDAQ-100 companies.
The corpus comprises 162{,}311 documents across four modalities:
\textit{News} (3{,}143 articles),
\textit{FinancialTable} (35{,}038 indicator records),
\textit{StockPrice} (124{,}130 price records), and
\textit{Chart} (vision-based, excluded from our text-only evaluation).
Each question references specific date ranges (e.g., ``from 2022-06-27 to 2022-09-23'') and gold source document UUIDs for provenance evaluation.
Question subtasks include \textit{Extraction}, \textit{Calculation}, \textit{Sentiment}, and \textit{Trend} analysis.
To enable feasible evaluation, we use a stratified sample of 100 questions and a reduced corpus of 908 documents, comprising 227 gold sources and 681 randomly sampled non-gold documents.

%% file: appendices/licenses.tex
% ---------------------------------------------------------------
%  Appendix: Licenses and Terms of Use
% ---------------------------------------------------------------

\section{Licenses and Terms of Use}
\label{app:licenses}

The following terms are those stated in the official upstream releases at the time of access. ``Not specified'' means that we found no explicit license in the official repository or dataset card; it should not be interpreted as permission for unrestricted reuse. We use all third-party resources only for the research reported here. Our release does not relicense these resources, and users should obtain them from the linked upstream sources and comply with their terms. A repository license also does not override any separate rights in third-party source material contained in a dataset.

\paragraph{Datasets.}
\begin{itemize}[leftmargin=1.5em,itemsep=1pt,topsep=2pt]
    \item \href{https://doi.org/10.7910/DVN/28075}{\textbf{ICEWS05-15}}: The source ICEWS Coded Event Data is available for research and educational use only; commercial use is prohibited. We use the derived ICEWS05-15 split for gate pretraining and RQ1 evaluation.
    \item \href{https://github.com/czy1999/MultiTQ}{\textbf{MultiTQ}}: Not specified. Its ICEWS-based knowledge graph remains subject to the ICEWS terms above.
    \item \href{https://github.com/snap-research/locomo/blob/main/LICENSE.txt}{\textbf{LoCoMo}}: Creative Commons Attribution--NonCommercial 4.0 International (CC BY-NC 4.0).
    \item \href{https://huggingface.co/datasets/MemGPT/MSC-Self-Instruct}{\textbf{DMR-MSC}}: Apache License 2.0, as stated by the official MSC-Self-Instruct dataset card.
    \item \href{https://github.com/lijunfeng99/FinTMMBench/blob/main/LICENSE}{\textbf{FinTMMBench}}: MIT License for the official repository release; rights in any underlying third-party source material remain with the respective providers.
\end{itemize}

\paragraph{Agentic-memory baselines.}
\begin{itemize}[leftmargin=1.5em,itemsep=1pt,topsep=2pt]
    \item \href{https://github.com/mem0ai/mem0/blob/main/LICENSE}{\textbf{Mem0}}: Apache License 2.0.
    \item \href{https://github.com/getzep/graphiti/blob/main/LICENSE}{\textbf{Zep/Graphiti}}: Apache License 2.0. We evaluate Graphiti, the open-source temporal graph engine associated with Zep, rather than the managed Zep service.
    \item \href{https://github.com/WujiangXu/A-mem/blob/main/LICENSE}{\textbf{A-Mem}}: MIT License.
    \item \href{https://github.com/EverM0re/LiCoMemory}{\textbf{LiCoMemory}}: Not specified. We therefore do not treat its public repository as permission to redistribute its source code.
    \item \href{https://github.com/OSU-NLP-Group/HippoRAG/blob/main/LICENSE}{\textbf{HippoRAG}}: MIT License.
\end{itemize}

The TKGE comparison values in Table~\ref{tab:tkge_results} are taken from the cited papers, as stated in its caption, rather than from redistributed third-party implementations. Model and API access (including GPT-5-mini, text-embedding-3-small, LLaMA-3.1-70B, and BGE-M3) remains subject to the corresponding providers' terms.

%% file: appendices/metrics.tex
% ---------------------------------------------------------------
%  Appendix: Evaluation Metrics and Answer Verification
% ---------------------------------------------------------------

\section{Evaluation Metrics and Answer Verification}
\label{app:metrics}

\paragraph{Temporal KG Completion (RQ1).}
For ICEWS05-15, we report Mean Reciprocal Rank (MRR) and Hits@$k$ for $k \in \{1, 3, 10\}$, following the standard filtered setting:
\begin{equation}
\text{MRR} = \frac{1}{|Q|} \sum_{q \in Q} \frac{1}{\text{rank}(q)}
\end{equation}
\begin{equation}
\text{Hits@}k = \frac{1}{|Q|} \sum_{q \in Q} \mathbb{1}[\text{rank}(q) \leq k]
\end{equation}

\paragraph{Agentic Memory Retrieval (RQ2, RQ3).}
For MultiTQ and DMR-MSC, we report MRR and Hits@$k$ ($k \in \{1, 3, 10\}$) computed over retrieved facts, where each query has a single gold answer:
\begin{equation}
\text{Hits@}k = \mathbb{1}[\exists\, i \leq k : \text{doc}_i \in \mathcal{G}]
\end{equation}
For LoCoMo, the main table reports Recall@10 per question category, computed as the fraction of gold evidence passages found in the top-10 retrieved documents:
\begin{equation}
\text{Recall@}k = \frac{|\{d \in \text{top-}k : d \in \mathcal{G}\}|}{|\mathcal{G}|}
\end{equation}
For FinTMMBench, we report Recall@$k$ for $k \in \{1, 3, 5, 10\}$ and MRR, computed against gold source document UUIDs. Appendix~\ref{app:additional-evaluation} additionally reports aggregate LoCoMo partial-match F1.

\paragraph{Answer Quality.}
We evaluate downstream answer quality using LLM@$k$ Accuracy, where $k$ denotes the number of retrieved documents provided to the answer LLM.
For DMR-MSC and FinTMMBench, a generated answer is scored as correct via the two-stage LLM judge pipeline described in Appendix~\ref{app:llm-judge}.
For MultiTQ, we instead use a rule-based cascading verifier (Appendix~\ref{app:multitq-verifier}) following the original benchmark protocol.
We report LLM@$k$ Accuracy at context sizes $k \in \{5, 10\}$.
Appendix~\ref{app:additional-evaluation} re-evaluates the saved DMR-MSC and FinTMMBench outputs with a stricter conflict-aware judge and reports a manual audit of unresolved alternatives.

% ---------------------------------------------------------------
\subsection{LLM Judge}
\label{app:llm-judge}

We use a two-stage answer evaluation pipeline.
\textbf{Stage~1} (fast path): if the lowercased, whitespace-normalised gold answer is a substring of the generated answer, the answer is immediately labelled \textsc{correct}.
\textbf{Stage~2} (LLM fallback): for non-matching answers, we invoke GPT-5.2 as a judge using the prompt shown in Box~\ref{box:judge-prompt}.
The judge is called with \texttt{temperature=0} and JSON response format $\{$\texttt{"label": "CORRECT"|"WRONG"}$\}$.

\stepbox{box:judge-prompt}
\begin{promptbox}{LLM Judge Prompt}
\begin{lstlisting}[
basicstyle=\ttfamily\small,
breaklines=true,
breakindent=0pt,
columns=fullflexible,
keepspaces=true,
showstringspaces=false
]
Your task is to label an answer to a question as `CORRECT' or `WRONG'. You will be given:
    (1) a question (posed by one user to another user),
    (2) a `gold' (ground truth) answer,
    (3) a generated answer
which you will score as CORRECT/WRONG.

The point of the question is to ask about something one user should know about the other user based on their prior conversations. The gold answer will usually be a concise and short answer that includes the referenced topic, for example:

Question: Do you remember what I got the last time I went to Hawaii?
Gold answer: A shell necklace

The generated answer might be much longer, but you should be generous with your grading - as long as it touches on the same topic as the gold answer, it should be counted as CORRECT.

For time related questions, the gold answer will be a specific date, month, year, etc. The generated answer might be much longer or use relative time references (like "last Tuesday" or "next month"), but you should be generous with your grading - as long as it refers to the same date or time period as the gold answer, it should be counted as CORRECT. Even if the format differs (e.g., "May 7th" vs "7 May"), consider it CORRECT if it's the same date.

Now it's time for the real question:
Question: {question}
Gold answer: {gold_answer}
Generated answer: {generated_answer}

First, provide a short (one sentence) explanation of your reasoning, then finish with CORRECT or WRONG. Do NOT include both CORRECT and WRONG in your response, or it will break the evaluation script.

Just return the label CORRECT or WRONG in a json format
with the key as "label".
\end{lstlisting}
\end{promptbox}

% ---------------------------------------------------------------
\subsection{MultiTQ Answer Verifier}
\label{app:multitq-verifier}

For MultiTQ, answer correctness is determined by a cascading multi-strategy rule-based verifier (rather than the LLM judge used for other benchmarks), adopted from \citet{tan2026memotimememoryaugmentedtemporalknowledge} to handle the diverse answer formats in temporal KGQA (entities, dates at varying granularities, and multi-part answers).
The strategies are applied in order; the first match determines the verdict:
\begin{enumerate}[leftmargin=2em]
    \item \textbf{Exact match}: normalised entity string equality.
    \item \textbf{Containment}: bidirectional substring check (gold $\subseteq$ prediction or prediction $\subseteq$ gold).
    \item \textbf{Advanced normalisation}: strip prefixes (e.g., ``The''), brackets, and punctuation, then substring match.
    \item \textbf{Time format matching}: year-month level matching with month-name support (e.g., ``January 2013'' $\approx$ ``2013-01'').
    \item \textbf{Multi-answer}: comma-separated answer parts are matched individually.
    \item \textbf{Semantic overlap}: word overlap $>50\%$ between prediction and gold answer tokens.
    \item \textbf{Loose match}: remove all spaces and underscores, then substring match.
\end{enumerate}

%% file: appendices/additional-evaluation.tex
\section{Additional Evaluation Analyses}
\label{app:additional-evaluation}

\subsection{Component Ablation Details}
\label{app:component-ablation-details}

Table~\ref{tab:component-ablation-full} reports all metrics for the MultiTQ ablation summarised in Table~\ref{tab:component-ablation}. Every variant uses the same fixed set of 500 questions, the complete graph of 461{,}329 facts, cached embeddings, ChronoR scoring dimension, 200 training epochs, reader, verifier, and optimisation settings.

\begin{table*}[t]
\centering
\small
\caption{Full single-factor ablation on MultiTQ. ``No gate'' uses one temporal speed for every relation; ``discrete time'' replaces functional time with learned timestamp lookups.}
\label{tab:component-ablation-full}
\begin{tabular}{lrrrrrr}
\toprule
\textbf{Variant} & \textbf{MRR} & \textbf{H@1} & \textbf{H@3} & \textbf{H@10} & \textbf{Acc@5} & \textbf{Acc@10} \\
\midrule
Full \ourapproach & \textbf{0.3371} & \textbf{0.252} & \textbf{0.384} & 0.502 & \textbf{0.366} & \textbf{0.392} \\
w/o time contrastive & 0.3254 & 0.238 & 0.368 & \textbf{0.504} & 0.340 & 0.390 \\
w/o semantic gate & 0.1044 & 0.084 & 0.116 & 0.142 & 0.100 & 0.108 \\
Discrete time & 0.0452 & 0.042 & 0.046 & 0.048 & 0.040 & 0.040 \\
\bottomrule
\end{tabular}
\end{table*}

\subsection{LoCoMo Answer-Level F1}

Table~\ref{tab:locomo-f1} reports partial-match F1 over all 1{,}986 LoCoMo questions, using the same reader outputs as the retrieval evaluation. Although \ourapproach improves Recall@10 in Table~\ref{tab:locomo_sub}, it matches HippoRAG at 0.457 F1 in the OpenAI configuration and improves F1 by only 0.001 in the server configuration. The original LoCoMo study reports a similar gap between evidence recall and answer F1, which it attributes to the reader's difficulty in using a larger or noisier context~\citep{maharana-etal-2024-evaluating}.

\begin{table}[h]
\centering
\small
\caption{Aggregate LoCoMo partial-match F1. ``OpenAI'' and ``Server'' refer to the graph-construction and embedding configurations; the answer reader is fixed.}
\label{tab:locomo-f1}
\begin{tabular}{lcc}
\toprule
\textbf{Method} & \textbf{OpenAI F1} & \textbf{Server F1} \\
\midrule
Zep & 0.343 & 0.320 \\
Mem0 & 0.457 & 0.444 \\
A-Mem & \textbf{0.459} & 0.382 \\
LicoMemory & 0.346 & 0.306 \\
HippoRAG & 0.457 & 0.446 \\
\ourapproach & 0.457 & \textbf{0.447} \\
\bottomrule
\end{tabular}
\end{table}

\subsection{Conflict-Aware Judge Robustness}

We use the answer-level judge only for DMR-MSC and FinTMMBench. ICEWS05-15, MultiTQ retrieval, MultiTQ answer accuracy, and LoCoMo retrieval use deterministic or gold-evidence metrics. To test whether the wording of the original judge accepts unresolved contradictions, we evaluated the same saved answers with a stricter prompt. We did not rerun retrieval or answer generation. The stricter judge requires the answer to state the gold fact and marks unresolved conflicting alternatives as incorrect.

As shown in Table~\ref{tab:strict-judge}, \ourapproach remains the top-ranked method on both benchmarks. The original and strict judges agree on 98.8\% of the 500 DMR-MSC answers and 98.0\% of the 100 FinTMMBench answers. We also inspected answers that contained conflicting candidates. Three \ourapproach answers on DMR-MSC contained unresolved alternatives, and the original judge accepted two of them (0.4\% of the evaluation set). We found no such case on FinTMMBench.

\begin{table*}[h]
\centering
\small
\caption{Acc@5 under the original and stricter conflict-aware answer judges. All saved answers are identical across judge protocols.}
\label{tab:strict-judge}
\begin{tabular}{lcccc}
\toprule
\textbf{Method} & \textbf{DMR-MSC original} & \textbf{DMR-MSC strict} & \textbf{FinTMMBench original} & \textbf{FinTMMBench strict} \\
\midrule
Zep & 0.302 & 0.278 & 0.480 & 0.470 \\
Mem0 & 0.858 & 0.834 & 0.550 & 0.540 \\
A-Mem & 0.848 & 0.830 & 0.540 & 0.540 \\
LicoMemory & 0.670 & 0.604 & 0.480 & 0.480 \\
HippoRAG & 0.852 & 0.832 & 0.550 & 0.550 \\
\ourapproach & \textbf{0.862} & \textbf{0.850} & \textbf{0.580} & \textbf{0.560} \\
\bottomrule
\end{tabular}
\end{table*}

%% file: appendices/extended-experimental-setup.tex
% ---------------------------------------------------------------
%  Appendix: Experimental Setup (configs + baselines + hyperparams)
% ---------------------------------------------------------------

\section{Implementation Configurations and Hyperparameters}
\label{app:impl-configs-and-hyperparameters}

\subsection{Configurations}
\label{app:impl-configs}

We evaluate all agentic memory benchmarks under two implementation configurations:

\begin{enumerate}[leftmargin=2em]
    \item \textbf{OpenAI}: GPT-5-mini for graph construction (NER + OpenIE) and text-embedding-3-small for embedding. This configuration tests performance with API-based models.
    \item \textbf{Server}: LLaMA-3.1-70B-Instruct~\citep{llama3} served via vLLM for graph construction and BAAI/BGE-M3~\citep{chen-etal-2024-m3} for embedding. This configuration demonstrates open-source reproducibility.
\end{enumerate}

In both configurations, the \textbf{answer LLM} and \textbf{LLM judge} always use GPT-5.2 via the OpenAI API to ensure fair comparison across all baselines.

\paragraph{Baseline systems.}
We compare against:
\begin{itemize}[leftmargin=2em]
    \item \textbf{Mem0}~\citep{chhikara2025mem0buildingproductionreadyai}: FAISS-based vector memory with per-document embedding and search.
    \item \textbf{Zep}~\citep{rasmussen2025zeptemporalknowledgegraph}: Temporal knowledge graph with Neo4j backend and entity extraction.
    \item \textbf{HippoRAG}~\citep{gutiérrez2024hipporag,gutierrez2025hipporag2}: Knowledge graph-augmented RAG with Personalised PageRank retrieval and Neo4j backend. \ourapproach builds upon HippoRAG's graph construction pipeline.
\end{itemize}

All baselines use the same answer LLM, LLM judge, and evaluation metrics for fair comparison.
We use the same temporal configuration for every target benchmark and do not tune hyperparameters separately for each dataset. The Semantic Speed Gate remains frozen; training on the target graph updates only the entity and relation embeddings and the global temporal-spectrum parameters.

\subsection{TKGE Hyperparameters}
\label{app:tkge-hyperparams}

Table~\ref{tab:tkge-hyperparams} reports the full TKGE hyperparameter configuration used across all experiments.

\begin{table*}[h]
\centering
\small
\caption{TKGE hyperparameters.
Parameters above the mid-rule are core KGE settings; below are temporal extension parameters.}
\label{tab:tkge-hyperparams}
\begin{tabular}{lcc}
\toprule
\textbf{Parameter} & \textbf{Value} & \textbf{Description} \\
\midrule
\texttt{temporal\_backbone} & \texttt{chronor} & ChronoR rotation backbone \\
\texttt{chronor\_k} & 3 & Number of rotation sub-spaces \\
\texttt{gamma} & 200.0 & Embedding range: $({\gamma + 2})/{\text{dim}}$ \\
\texttt{adversarial\_temperature} & 1.0 & Self-adversarial sampling temperature \\
\texttt{regularization\_weight} & $10^{-5}$ & N3 per-batch regularization (ChronoR) \\
\texttt{steps\_per\_update} & 500 & Training epochs per update cycle \\
\texttt{num\_conflict\_negatives} & 1 & Tails from same $(h, r)$ group \\
\midrule
\texttt{time\_source} & \texttt{happen} & Use \texttt{text\_time} for temporal scoring \\
\texttt{time\_loss\_type} & \texttt{listwise} & Distribution-matching time loss \\
\texttt{time\_contrastive\_weight} & 0.5 & Time-contrastive loss weight ($\lambda_{\text{time}}$) \\
\texttt{num\_time\_negatives} & 8 & Negative time samples per fact \\
\texttt{time\_sigma\_years} & 0.25 & Gaussian kernel $\sigma$ (years) \\
\texttt{time\_sigma\_years\_start} & 0.5 & Curriculum start $\sigma$ \\
\texttt{time\_sigma\_years\_end} & 0.02 & Curriculum end $\sigma$ \\
\texttt{time\_neg\_jitter\_years} & 0.02 & Temporal jitter for negatives \\
\texttt{time\_neg\_far\_days} & 365 & Far negative offset (days) \\
\texttt{time\_neg\_min\_days\_start} & 90 & Curriculum start min-gap (days) \\
\texttt{time\_neg\_min\_days\_end} & 3 & Curriculum end min-gap (days) \\
\texttt{time\_neg\_min\_days\_decay} & 60 & Min-gap curriculum decay (epochs) \\
\texttt{tkge\_weight} ($\alpha_g$) & 0.3 & Temporal reranking strength \\
\bottomrule
\end{tabular}
\end{table*}

%% file: appendices/math-analysis.tex
\section{Theoretical Analysis}
\label{app:math-analysis}

\subsection{Overview}

While transitioning from a discrete timestamp dictionary to a continuous functional rotation resolves the granularity and extrapolation limitations of traditional temporal models, it introduces distinct algebraic and computational challenges. This appendix section provides the rigorous theoretical justification for our framework, structured around three core aspects:

\begin{itemize}[leftmargin=1.5em, itemsep=4pt, topsep=4pt, parsep=0pt]
    \item \textbf{Complex vs. Real Representation (\ref{app:math-analysis-complex-vs-real}):} We clarify the relationship between the complex-space formulation in $\mathbb{C}^d$ and its real-space equivalent in $\mathbb{R}^{2d}$, demonstrating why the complex unitary group offers concrete structural advantages for temporal knowledge graphs.
    \item \textbf{Retrieval Reformulation (\ref{app:math-analysis-retrieval-reformulation}):} We prove that the orthogonal structure of the rotation operator allows the continuous temporal transformation to be isolated entirely on the query side, thereby preserving strict compatibility with static vector search indices.
    \item \textbf{Temporal Interpolation (\ref{app:math-analysis-temporal-interpolation}):} We analyse a simplified pairwise setting and give a sufficient condition for a unique, monotonic crossover between historical anchors.
\end{itemize}

%============================================================================================================
\subsection{Complex vs. Real Representation}
\label{app:math-analysis-complex-vs-real}

Our methodology embeds entities in $\mathbb{R}^{2d}$ and interprets them as complex vectors in $\mathbb{C}^d$. These two views are algebraically equivalent: a rotation by angle $\theta$ in $\mathbb{C}^d$ corresponds to applying a block-diagonal orthogonal matrix $R_\theta \in \mathbb{R}^{2d \times 2d}$ composed of $d$ independent $2 \times 2$ rotation blocks.

The complex formulation is not merely notational convenience. As established by RotatE~\citep{DBLP:conf/iclr/SunDNT19} and ChronoR~\citep{sadeghian2021chronor}, expressing relational transformations as element-wise rotations (Hadamard products) in the unitary group $U(1)^d$ naturally captures critical graph structures such as symmetry, antisymmetry, inversion, and composition, while reducing transformation cost from $\mathcal{O}(d^2)$ (general matrix multiplication) to $\mathcal{O}(d)$ (element-wise operations). Our analysis below establishes that these algebraic benefits seamlessly extend to continuous time dynamics.

%============================================================================================================
\subsection{Retrieval Reformulation and Static Index Compatibility}
\label{app:math-analysis-retrieval-reformulation}

In the temporal knowledge graph retrieval setting, evaluating a query $(h,r,?)$ against all candidate tail entities $t \in \mathcal{E}$ under continuous time introduces a severe scalability bottleneck. Naively applying the time-dependent rotation to the entire candidate vocabulary dictates an $\mathcal{O}(N \cdot d)$ dynamic transformation at query time, which fundamentally breaks compatibility with prebuilt static vector search indices and renders large-scale querying computationally intractable. 

We ask whether continuous temporal queries can be evaluated without changing the candidate index. We use the real-space implementation from \S\ref{sec:rotation}, where $R_{\boldsymbol{\theta}_r(\tau)} \in \mathbb{R}^{2d \times 2d}$ is the block-diagonal orthogonal matrix corresponding to $\mathrm{Rot}(\cdot,\boldsymbol{\theta}_r(\tau))$, and $\mathbf{W}_r = \mathrm{diag}(\mathbf{w}_r \odot \hat{\mathbf{w}}_r) \in \mathbb{R}^{2d \times 2d}$ is the relation-specific diagonal scaling matrix. For clarity, we omit the component index $c$ below. The same reformulation applies to each of the $k$ components before their scores are summed. It moves the temporal transformation to the query side, allowing the candidate index to remain unchanged and compatible with Maximum Inner Product Search (MIPS).

\begin{proposition}[Query-Side Reformulation of Temporal Retrieval]
\label{prop:reformulation}
Let $\mathbf{C} \in \mathbb{R}^{N \times 2d}$ be the static matrix of candidate embeddings in the real-space implementation, where the $t$-th row of $\mathbf{C}$ is $\mathbf{e}_t^\top$. For a continuous timestamp $\tau$, suppose the score against candidate $t$ is defined by the standard Euclidean inner product:
$$s_{\mathrm{kge}}((h,r,t)\mid\tau) = \langle \mathbf{W}_r R_{\boldsymbol{\theta}_r(\tau)} \mathbf{e}_h, R_{\boldsymbol{\theta}_r(\tau)} \mathbf{e}_t \rangle$$
where $R_{\boldsymbol{\theta}_r(\tau)}\in\mathbb{R}^{2d\times 2d}$ is a block-diagonal orthogonal matrix representing the phase shift, and $\mathbf{W}_r\in\mathbb{R}^{2d\times 2d}$ is the diagonal relation-specific scaling operator. Then there exists a candidate-independent query vector $\mathbf{q}(\tau) \in \mathbb{R}^{2d}$ defined as:
$$\mathbf{q}(\tau) = R_{-\boldsymbol{\theta}_r(\tau)} \mathbf{W}_r R_{\boldsymbol{\theta}_r(\tau)} \mathbf{e}_h$$
such that the score reduces to:
$$s_{\mathrm{kge}}((h,r,t)\mid\tau) = \langle \mathbf{q}(\tau), \mathbf{e}_t \rangle$$
for all candidates $t$. Consequently, the exact 1-vs-$N$ retrieval reduces to an $\mathcal{O}(d)$ query-side preprocessing step followed by an $\mathcal{O}(N \cdot d)$ static inner-product search over $\mathbf{C}$.
\end{proposition}

\begin{proof}
In a direct candidate-side implementation, evaluating the score requires applying the time-dependent rotation $R_{\boldsymbol{\theta}_r(\tau)}$ to every candidate vector $\mathbf{e}_t$ at query time. Although this can be computed in a streaming fashion in $\mathcal{O}(N \cdot d)$ time, this query-time temporal transformation fundamentally prevents the use of any prebuilt static inner-product index.

However, since the rotation matrix $R_{\boldsymbol{\theta}_r(\tau)}$ is orthogonal, its transpose satisfies $R_{\boldsymbol{\theta}_r(\tau)}^\top = R_{-\boldsymbol{\theta}_r(\tau)}$. Utilizing the adjoint property of the Euclidean inner product, we can strictly transfer the rotation from the candidate vector back to the query side:
$$s_{\mathrm{kge}}((h,r,t)\mid\tau) = \langle R_{-\boldsymbol{\theta}_r(\tau)} \big( \mathbf{W}_r R_{\boldsymbol{\theta}_r(\tau)} \mathbf{e}_h \big), \mathbf{e}_t \rangle$$

By substituting the definition of the unrotated query vector $\mathbf{q}(\tau)$, the scoring function trivially simplifies to $s_{\mathrm{kge}}((h, r, t) \mid \tau) = \langle \mathbf{q}(\tau), \mathbf{e}_t \rangle$. We note that the temporal dependence remains nontrivial whenever the relation-specific operator $\mathbf{W}_r$ does not commute with $R_{\boldsymbol{\theta}_r(\tau)}$.

\textbf{Complexity Analysis.}
Constructing the query vector $\mathbf{q}(\tau)$ requires $\mathcal{O}(d)$ trigonometric evaluations and $\mathcal{O}(d)$ structured linear operations, since $R_{\boldsymbol{\theta}_r(\tau)}$ is block-diagonal and $\mathbf{W}_r$ is a diagonal matrix. To evaluate the query against all $N$ candidates simultaneously, we compute the full score vector $\mathbf{s}(\tau) \in \mathbb{R}^N$, where each $t$-th element strictly corresponds to the individual scalar score $s_{\mathrm{kge}}((h,r,t)\mid\tau)$. This is achieved via a single matrix-vector multiplication:
$$\mathbf{s}(\tau) = \mathbf{C} \, \mathbf{q}(\tau)$$
which requires $\mathcal{O}(N \cdot d)$ arithmetic operations. The additional query-time workspace is $\mathcal{O}(d)$, and the retrieval stage requires no candidate-side temporal transformation. This proves the proposition.
\end{proof}

\textit{Remark:} Because the final formulation reduces retrieval to a standard inner-product evaluation over a static candidate matrix $\mathbf{C}$, the model is directly compatible with exact or approximate inner-product search libraries (e.g., FAISS) without requiring index rebuilds across queries.

%============================================================================================================

\subsection{A Stylised Analysis of Pairwise Temporal Interpolation}
\label{app:math-analysis-temporal-interpolation}

A fundamental advantage of formulating time as a continuous rotation operator lies in its structural capacity to interpolate knowledge between historical observations. While exact global retrieval over an entire knowledge graph entails complex multi-dimensional phase interference, we can rigorously demonstrate the model's interpolation mechanics by analysing a stylised pairwise regime. 

\textbf{Trigonometric Isomorphism of the Scoring Function.} 
Before stating the formal proposition, we establish the algebraic bridge between the global inner-product scoring formulation (Equation~\eqref{eq:scoring-function}) and its corresponding trigonometric expansion. Since the total score $s_{\mathrm{kge}}$ is a linear sum of independent dimensional contributions ($s_{\mathrm{kge}} = \sum_j s_j$), we can isolate the temporal dynamics within a single complex dimension. For each $2\text{D}$ rotational subspace $j$, the partial score evaluates a bilinear form:
$$s_j(\tau) = \mathbf{h}_j^\top R_{\theta_j(\tau)}^\top \mathbf{W}_j R_{\theta_j(\tau)} \mathbf{t}_j$$

where $\mathbf{h}_j, \mathbf{t}_j \in \mathbb{R}^2$ are the static entity vectors, $\mathbf{W}_j = \mathrm{diag}(w_{j,1}, w_{j,2})$ is the relation weight matrix, and $\theta_j(\tau)$ is the $j$-th scalar component of the relation-specific rotation vector $\boldsymbol{\theta}_r(\tau)$ (defined in Equation~\eqref{eq:theta_r}). Because the relation weights typically break rotational symmetry ($w_{j,1} \neq w_{j,2}$), expanding $R_{\theta_j(\tau)}^\top \mathbf{W}_j R_{\theta_j(\tau)}$ via double-angle identities yields a linear combination of trigonometric functions:
$$s_j(\tau) = C_j + A_j \cos(2\theta_j(\tau)) + B_j \sin(2\theta_j(\tau))$$
where $C_j, A_j, B_j$ are time-independent constants determined by the entities and relation weights. By applying the harmonic addition theorem, this combination can be exactly re-parameterised as a phase-shifted cosine wave:
$$s_j(\tau) = C_j + \gamma_j \cos(2\theta_j(\tau) - \phi_j)$$
with amplitude $\gamma_j = \sqrt{A_j^2 + B_j^2} > 0$ and phase shift $\phi_j = \operatorname{atan2}(B_j,A_j)$. Since the rotation angle $\theta_j(\tau)$ is linearly proportional to time $\tau$ (Equation~\eqref{eq:theta_r}), summing these independent subspaces across all dimensions reduces the exact relational inner product to a multi-frequency cosine expansion. This form motivates the simplified scoring model below.

In this section, we establish a sufficient condition under which the temporal competition between two consecutive facts exhibits a smooth, monotonic crossover, thereby yielding a deterministic decision boundary for unobserved intermediate timestamps.

\begin{proposition}[Sufficient Condition for Monotone Pairwise Crossover]
\label{prop:interpolation}
Consider a temporal query $(h, r)$ with two mutually exclusive facts, $A$ and $B$, observed at consecutive timestamps $T$ and $T+t$ ($t > 0$). Assume that over the interpolation interval $[T, T+t]$, the expansion of the inner-product scoring function (Equation~\eqref{eq:scoring-function}) for each candidate $c \in \{A, B\}$ is governed by a stylised synchronised cosine expansion:
$$s_c(\tau) = \sum_{j=1}^{d} \gamma_{c,j} \cos\big(\tilde{\omega}_j (\tau - \tau_c)\big), \quad \gamma_{c,j} > 0$$
where $\tau_A = T$ and $\tau_B = T+t$ denote the local phase-alignment peaks. Consistent with the double-angle expansion above and Equation~\eqref{eq:theta_r}, $\tilde{\omega}_j = 2s \cdot \alpha_r \cdot \omega_j$ represents the effective angular velocity, which incorporates the global time scale $s$, inverse frequency $\omega_j$, and relation-specific semantic speed gate $\alpha_r$.

Assume the model accurately reconstructs these historical anchors such that $s_A(T) > s_B(T)$ and $s_A(T+t) < s_B(T+t)$. Provided the effective angular velocities satisfy the strict half-period bound $\tilde{\omega}_j \in (0, \frac{\pi}{t}]$ for all $j$, the pairwise confidence gap $\Delta s(\tau) = s_A(\tau) - s_B(\tau)$ is strictly monotonically decreasing on the open interval $(T, T+t)$. 

Consequently, there exists a unique crossover timestamp $\tau^* \in (T, T+t)$ satisfying $s_A(\tau^*) = s_B(\tau^*)$. For any intermediate time $\tau \in (T, T+t)$, the model strictly prefers $A$ when $\tau < \tau^*$, strictly prefers $B$ when $\tau > \tau^*$, and yields an exact pairwise tie at $\tau^*$.
\end{proposition}

\begin{proof}
By the assumption of anchor correctness, the relative confidence at the boundaries satisfies $\Delta s(T) > 0$ and $\Delta s(T+t) < 0$.

To analyse the transition mechanics, we differentiate the confidence gap with respect to continuous time $\tau$:
$$\frac{d}{d\tau} \Delta s(\tau) = s_A'(\tau) - s_B'(\tau)$$

The derivatives of the stylised scoring functions are given by:
$$s_A'(\tau) = -\sum_{j=1}^{d} \gamma_{A,j} \tilde{\omega}_j \sin\big(\tilde{\omega}_j (\tau - T)\big)$$
$$s_B'(\tau) = -\sum_{j=1}^{d} \gamma_{B,j} \tilde{\omega}_j \sin\big(\tilde{\omega}_j (\tau - (T+t))\big)$$

For any intermediate time $\tau = T + dt$ with $dt \in (0, t)$, the temporal displacement for candidate $A$ is $dt$. Given $\tilde{\omega}_j \in (0, \frac{\pi}{t}]$, the phase argument $\tilde{\omega}_j dt$ strictly resides in $(0, \pi)$. In this interval, the sine function is strictly positive. Since $\gamma_{A,j} > 0$ and $\tilde{\omega}_j > 0$, it strictly follows that $s_A'(\tau) < 0$. 

Conversely, the temporal displacement for candidate $B$ is $dt - t < 0$. Under the same frequency bound, the phase argument $\tilde{\omega}_j (dt - t)$ strictly resides in $(-\pi, 0)$, where the sine function is strictly negative. This renders $s_B'(\tau) > 0$.

Therefore, the derivative of the pairwise gap is strictly negative across the entire open interval:
$$\frac{d}{d\tau} \Delta s(\tau) < 0 \quad \forall \tau \in (T, T+t)$$

This strict monotonicity, coupled with the continuity of the scoring functions on $[T, T+t]$ and the boundary conditions, guarantees via the Intermediate Value Theorem (IVT) the existence of exactly one root $\tau^*$ in $(T, T+t)$ where $\Delta s(\tau^*) = 0$.
\end{proof}

\textit{Remark:} The proposition gives a sufficient condition for pairwise interpolation; it does not guarantee a single crossover for the complete learned spectrum or the full candidate set. Under the half-period bound, the two candidates cross once between consecutive anchors. The crossover $\tau^*$ need not occur at the midpoint $T+t/2$, because its position depends on the amplitudes $\gamma_{A,j}$ and $\gamma_{B,j}$. Higher frequencies may violate the bound and produce local oscillations, as shown in Appendix~\ref{app:qualitative}.

%% file: appendices/qualitative-examples.tex
\section{Qualitative Analysis: Geometric Shadowing in Action}
\label{app:qualitative}

To make the geometric shadowing mechanism concrete, we present scoring traces from a controlled experiment.
We select a small subset of real temporal triples from ICEWS05-15~\citep{garcia-duran-etal-2018-learning}, including 4 \texttt{(Obama, Consult, Blair)} facts timestamped between June 2007 and April 2008, and 6 \texttt{(Obama, Consult, Xi Jinping)} facts timestamped between June 2013 and September 2015, along with static facts (\texttt{born in}) and auxiliary relation slots (\texttt{Make a visit}, \texttt{Express intent to cooperate}).
These additional facts are included to provide sufficient training signal for the shared entity embeddings and frequency spectrum while training on the two competing facts alone leaves the model severely underconstrained, producing degenerate oscillations. We train a \ourapproach-ChronoR model on this subset with the bundled pretrained gate and sweep the query timestamp $\tau_q$ to observe how candidate scores evolve continuously.

\subsection{Dynamic Relation: Score Crossover}

\begin{figure*}[htbp]
    \centering
    \includegraphics[width=0.75\linewidth]{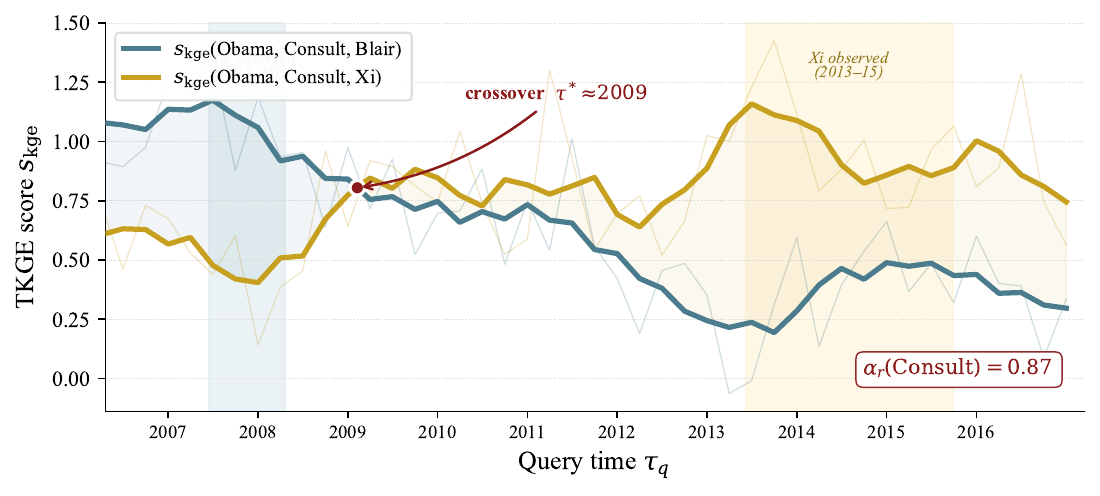}
    \caption{Scores for the competing slot \texttt{(Obama, Consult, ?)}. The light curves are the raw quarterly scores used for retrieval, and the bold curves are 5-quarter rolling averages shown only to make the overall trend easier to see. The blue and yellow regions mark the observation periods for Blair (2007-2008) and Xi (2013-2015), respectively. The red dot marks where the smoothed curves cross. The gate assigns a high rotation speed to ``Consult'' ($\alpha_r=0.87$).}
    \label{fig:score-crossover}
\end{figure*}

Consider the relational slot \texttt{(Barack Obama, Consult, ?)} with two competing tail entities: \textit{Tony Blair} (4 observations, 2007-06 to 2008-04) and \textit{Xi Jinping} (6 observations, 2013-06 to 2015-09). The pretrained Semantic Speed Gate assigns $\alpha_r(\texttt{Consult}) = 0.87$, which is consistent with the rapidly changing participants of this relation.

Figure~\ref{fig:score-crossover} shows the TKGE scores $s_{\mathrm{kge}}$ as the query time $\tau_q$ moves from 2007 to 2016. Blair has the higher smoothed score at the beginning of this period, but his score gradually falls as Xi's rises. The smoothed curves cross around 2009, after which Xi \emph{shadows} Blair. Proposition~\ref{prop:interpolation} guarantees a unique crossover only when its half-period condition holds for every frequency component.

Crucially, neither fact is deleted: both remain in the append-only memory, and the rotation operator continuously modulates their alignment with the query time.

\paragraph{On raw oscillations and multiple crossings.}
The raw quarterly scores (light curves in Figure~\ref{fig:score-crossover}) oscillate and cross more than once. Thus, not every learned frequency in this example satisfies the sufficient condition in Proposition~\ref{prop:interpolation}. The proposition guarantees monotonicity only when all frequencies satisfy the half-period bound $\tilde{\omega}_j \leq \pi/t$.
Two factors explain these fluctuations.
First, the model learns a spectrum of $kd$ frequency components, and higher-frequency components (those with period shorter than the ${\sim}5$-year observation gap) violate this bound, producing local oscillations.
Second, entity embeddings are shared across all relation slots.
Obama's embedding is jointly trained on \texttt{Consult}, \texttt{Make a visit}, and \texttt{born in} facts, so temporal dynamics from other slots introduce additional phase interference into the \texttt{Consult} scores. For instance, Blair's raw score briefly rises around 2010-2011 before resuming its decline, and Xi's score dips near 2015 before recovering.

The 5-quarter rolling average summarises the overall trend and produces a single crossover of the form analysed in Proposition~\ref{prop:interpolation}. The raw scores show that the proposition is not a guarantee for the complete learned model. The rolling average is used only in this figure, not during retrieval.

\subsection{Gate Inspection: Learned Relational Volatility}

Table~\ref{tab:gate-values} reports $\alpha_r$ for representative relations that either appeared during ICEWS05-15 pretraining (\textit{seen}) or were evaluated only after pretraining (\textit{unseen}). The values follow the expected volatility ordering without manual static/dynamic labels. These examples are qualitative and do not measure calibration over all relations.

\begin{table}[ht]
\centering
\small
\caption{Pretrained semantic speed gate values. Higher $\alpha_r$ = faster rotation. \textit{Seen}: appeared in ICEWS05-15 during pretraining; \textit{Unseen}: zero-shot via text embedding similarity.}
\label{tab:gate-values}
\begin{tabular}{lcc}
\toprule
\textbf{Relation} & $\alpha_r$ & \textbf{Category} \\
\midrule
\multicolumn{3}{l}{\textit{Seen during gate pretraining (ICEWS05-15)}} \\
\midrule
Consult                      & 0.87 & Dynamic  \\
Host a visit                 & 0.86 & Dynamic  \\
Engage in negotiation        & 0.63 & Dynamic  \\
Sign formal agreement        & 0.53 & Dynamic  \\
Cooperate economically       & 0.16 & Low      \\
Cooperate militarily         & 0.09 & Low      \\
\midrule
\multicolumn{3}{l}{\textit{Unseen (zero-shot via text embeddings)}} \\
\midrule
met with                     & 0.71 & Dynamic  \\
visited                      & 0.64 & Dynamic  \\
negotiated with              & 0.62 & Dynamic  \\
CEO of                       & 0.44 & Moderate \\
capital of                   & 0.36 & Moderate \\
species                      & 0.22 & Static   \\
citizen of                   & 0.17 & Static   \\
\bottomrule
\end{tabular}
\end{table}

A notable property of this result is that ICEWS05-15 contains \textit{few semantically static relations}. 
All 251 relation types describe political events (consulting, visiting, threatening, etc.), lacking permanent properties like ``born in'' or ``species''.
Across 461K facts spanning 4{,}017 timestamps, the average $(h,r)$ slot has 3.08 distinct tail entities, which means every relation exhibits temporal variation.\footnote{57 of 251 relations technically have single-tail slots, but these are all rare event types ($\leq$42 facts each) that appear static mostly due to data sparsity, not semantic permanence (e.g., ``Attempt to assassinate,'' ``Demand mediation'').} Yet the gate still learns a meaningful volatility gradient within this event-driven spectrum: episodic interactions like ``Consult'' (0.87) and ``Host a visit'' (0.86) receive high $\alpha_r$, while sustained state-level conditions like ``Cooperate militarily'' (0.09) and ``Cooperate economically'' (0.16) receive lower values.

The unseen examples show that the gate transfers beyond the relation strings observed during pretraining. Episodic predicates receive higher values (``met with'' at 0.71 and ``visited'' at 0.64), while persistent properties receive lower values (``citizen of'' at 0.17 and ``species'' at 0.22). This transfer comes from using text embeddings rather than relation IDs. We evaluate only the relations and domains reported here, so these examples do not establish that the gate is calibrated for arbitrary predicates.

The DMR-MSC results in Table~\ref{tab:dmrmsc_sub} provide a separate preservation check: temporal reranking does not reduce performance on this largely non-temporal conversational benchmark.